\documentclass{article}

\usepackage{PRIMEarxiv}

\usepackage[utf8]{inputenc}
\usepackage[T1]{fontenc}    
\usepackage{commath}
\usepackage{amsfonts}
\usepackage{bbm}
\usepackage{breqn}
\usepackage{amsmath}
\usepackage{dsfont}
\usepackage{amssymb}
\usepackage{optidef}
\usepackage{mathtools}
\usepackage{float}
\usepackage{booktabs} 
\usepackage{makecell}
\usepackage{tikz}
\usetikzlibrary{automata,positioning}
\usepackage{hyperref}
\usepackage{graphicx}
\usepackage{wrapfig}
\usepackage{comment}
\usepackage[ruled,vlined]{algorithm2e}

\SetCommentSty{mycommfont}
\usepackage{amsthm}
\usepackage{dblfloatfix}

\usepackage[backend=biber, style=apa,]{biblatex}
\newenvironment{subroutine}[1][htb]
{
\begin{algorithm}[#1]%
}{\end{algorithm}}

\makeatletter
\newtheorem*{rep@theorem}{\rep@title}
\newcommand{\newreptheorem}[2]{%
\newenvironment{rep#1}[1]{%
 \def\rep@title{#2 \ref{##1}}%
 \begin{rep@theorem}}%
 {\end{rep@theorem}}}
\makeatother

\newtheorem{theorem}{Theorem}
\newtheorem{proposition}{Proposition}
\newtheorem{lemma}{Lemma}
\newtheorem{definition}{Definition}

\newtheorem{assumption}{Assumption}

\newreptheorem{proposition}{Proposition}
\newreptheorem{lemma}{Lemma}
\newreptheorem{theorem}{Theorem}

\title{Multi-Player Bandits Robust to Adversarial Collisions
}

\author{
  \textbf{Shivakumar Mahesh} \\
  University of Oxford \\
  \texttt{shivakumar.mahesh@spc.ox.ac.uk} \\
   \and 
  \textbf{Anshuka Rangi} \\
  University of California San Diego \\
  \texttt{arangi@ucsd.edu} \\
  \AND 
  \textbf{Haifeng Xu} \\
  University of Chicago \\
  \texttt{haifengxu@uchicago.edu} \\
  \and
  \textbf{Long	Tran-Thanh} \\
  University of Warwick \\
  \texttt{long.tran-thanh@warwick.ac.uk} \\
}

\begin{document}
\maketitle

\begin{abstract}
Motivated by cognitive radios, stochastic Multi-Player Multi-Armed Bandits has been extensively studied in recent years. In this setting, each player pulls an arm, and receives a reward corresponding to the arm if there is no collision, namely the arm was selected by one single player. Otherwise, the player receives no reward if collision occurs. In this paper, we consider the presence of malicious players (or attackers) who obstruct the cooperative players (or defenders) from maximizing their rewards, by deliberately colliding with them. We provide the first decentralized and robust algorithm \textsc{RESYNC} for defenders whose performance deteriorates gracefully as $\tilde{O}(C)$ as the number of collisions $C$ from the attackers increases. We show that this algorithm is order-optimal by proving a lower bound which scales as $\Omega(C)$. This algorithm is agnostic to the algorithm used by the attackers and agnostic to the number of collisions $C$ faced from attackers. 

\end{abstract}

\section{Introduction}


 Multi-Player Multi-Armed Bandits (MP-MAB) algorithms have found applications in distributed computing, social recommendation systems, federated learning, sensor networks, Internet of Things, web services and crowdsourcing systems. Typically, these variants involve a large number of players playing a bandit instance, and may or may not be communicating with each other. 
 The objective of the players, together as a group, is to maximize the collective reward.  The distributed nature of these applications makes the learning algorithms prone to attacks from malicious players (or attackers).

In this paper, we focus an important and widely studied decentralized MP-MAB setting motivated by cognitive radios (\cite{Jouini2009,animaAnand2010}). In this setting, an arm corresponds to a channel frequency, and the player selects a channel to transmit information, where the reward corresponding to this arm is given by the transmission quality of the channel. A key feature of this MP-MAB setting is that if multiple players choose the same arm (or channel) in a round, then collision occurs and all these players receive zero reward for that round. Additionally, the players receive feedback whether or not a collision occurred on the arm they pulled. Finally, in the decentralized case, the players are independent and cannot communicate with each other through dedicated communication channels. 

Decentralized MP-MAB setting has been widely studied in the absence of attackers (\cite{MEGA,MC,lugosiMeharabian2018,mageshVeeralli2019a,alaturLevy2019}). The optimal algorithms in the absence of attackers critically relies on the assumption that all the players are cooperative and execute the same algorithm. However, this assumption critically impairs the application of these algorithms. For example, in cognitive radio, a channel can be accessed by any player with a transmitter, and access is not restricted to cooperative players alone. Therefore, malicious players (or attackers) can obstruct cooperative players (or defenders) from maximizing reward by deliberately colliding with them. Further, the algorithms used by such attackers is unknown to the defenders. There has been limited focus on addressing this key issue and designing robust MP-MAB algorithms for the defenders that are agnostic to the algorithm used by attackers.  

While the goal of the defenders is to minimize their collective regret, the goal of the attackers is to force the defenders to incur linear regret while minimizing the number of adversarial collisions induced by them. To this end, at each round, each attacker may pull an arm, observe the reward and the feedback if the collision occurred corresponding to the arm. If the attacker chooses to not to pull any arm, then no reward and collision information is observed, also termed as ``staying quiet" by \cite{alaturLevy2019}. This work focuses on proposing algorithms for the defenders which are robust to these attackers. As a motivating example, one may consider the defenders as licensed spectrum users and the attackers as unlicensed ones. In this case, the licensed users wish to find the optimal allocation of arms, in the presence of interference from unlicensed users.
{
\renewcommand{\arraystretch}{2.5}
 \begin{table*}[!t]
\centering
  \begin{tabular}{|l|l|l|l|l|l|}\hline
    \textit{Feedback Model} & \textit{Algorithm/ Reference} & \textit{Prior Knowledge} & \textit{Regret under our Attack model} \\ \hline
    \makecell{Non-Distinguishable \\Collision Sensing}  & \makecell{\textsc{mc}, Proposition \ref{prop:mcattack} \\ \cite{MC}}   & $T,\Delta$       & \makecell{ $\mathbb{E}R_T = \Omega(T)$ \\ under $C = O\left(\frac{K}{\Delta^2}\log{(K^2T)}\right)$  \\}   \\  \hline
    \makecell{Non-Distinguishable \\Collision Sensing} & \makecell{\textsc{sic-mmab},  Proposition \ref{prop:sicmmabattack}\\ \cite{sicmmab}}  & $T$                    & \makecell{ $\mathbb{E}R_T = \Omega\left(\frac{K-N}{K}T\right)$ \\ under $C = O\left(K^2\log{T}\right)$\\}  \\ \hline
    \makecell{Non-Distinguishable \\Collision Sensing}  & \makecell{\textsc{sic-gt}, Proposition \ref{prop:sicgtattack} \\\cite{sicgt}}     & $T$                    & \makecell{ $\mathbb{E}R_T = \Omega(T)$ \\ under $C = O\left(K^2\log{T}\right)$\\}  \\ \hline
    \makecell{Non-Distinguishable \\Collision Sensing}  & \makecell{\textsc{cnj,cuj},  Appendix \ref{sec:attackjammers} \\\cite{jammers}} & $T,\Delta$                    & \makecell{ $\mathbb{E}R_T = \Omega\left(T\right)$ \\ under $C = \tilde{O}\left(\log{T}\right)$\\}  \\ \hline
    \makecell{Distinguishable \\Collision Sensing}      & \makecell{\textsc{cdj},  Appendix \ref{sec:attackjammers} \\\cite{jammers}}         & $T,\Delta$                   & \makecell{ $\mathbb{E}R_T = \Omega\left(T\right)$ \\ under $C = \tilde{O}\left(\log{T}\right)$\\}  \\ \hline
    \makecell{Non-Distinguishable \\Collision Sensing}  & \makecell{{\textsc{resync}}, Theorem \ref{thm:mainRESYNC} \\This work                  }          & $T,\Delta, n, j$ &  \makecell{ $\mathbb{E}R_T = O\bigg(CN^2 + CK\frac{\log(K^2T)}{\Delta^2}\bigg)$\\}  \\ \hline
    \makecell{Distinguishable \\Collision Sensing}      & \makecell{{\textsc{resync2}}, Theorem \ref{thm:RESYNC2}     \\ This work}                 & $T,\Delta$       & \makecell{ $\mathbb{E}R_T = O\bigg(CK + K\frac{\log(K^2T)}{\Delta^2}\bigg)$\\}  \\ \hline
    \makecell{ Non-Distinguishable \\ + Distinguishable  \\Collision Sensing}      & \makecell{{\textsc{lower bound}}, Theorem \ref{thm:lowerbound1}     \\ This work}                 & N/A       & \makecell{ $\mathbb{E}R_T = \Omega\left(\frac{N}{K}C\right)$ \\ }  \\\bottomrule
  \end{tabular}
      \caption{Summary of Contributions: Regret Bounds of Decentralized MP-MAB Algorithms under $C$ number of collisions from attackers. Here, $N$ denote the number of defenders, and $\Delta = {\mu_{(N)}-\mu_{(N+1)}}$ is the gap between the expected rewards of $N$-th and $(N+1)$-th best arms. We use $j$ to denote the prior knowledge where, each defender has access to a integer $j \in [N]$ distinct from every other defender (for more details refer Assumption \ref{assm:cmn-knwlg}).} \label{tab:comparisons} 
\end{table*}
}
 \subsection{Related work}
The concern of designing MP-MAB algorithms robust to malicious players (or attackers) has been raised in multiple works (\cite{securityCRN,MC}) but has only been studied under the assumption that the attacker behaviour is known in advance to the  defenders. For instance, \cite{jammersOld2015, jammers} consider attackers that follow a specific algorithm of learning and pulling the optimal set of arms, and construct robust defence algorithms (\textsc{cnj, cuj} and \textsc{cdj}) against these specific attackers. In this work, we show that these algorithms are not robust to an attacker with a different attack strategy. Namely, we show that the existing algorithms including these robust algorithms incur linear regret $\Omega(T)$ with only small $O(\log T)$ number of adversarial collisions from an attacker with a different strategy. 


\cite{sicgt} consider non-cooperative players whose incentives are to maximize their own individual rewards. They provide the first algorithm \textsc{sic-gt}, which is robust to selfish players i.e. following their algorithm is a $\epsilon$-Nash equilibrium. For performance guarantees, they assume that selfish players will not deviate from this particular $\epsilon$-Nash equilibrium where every player executes the same algorithm. This assumption may fail in practice since non-cooperative players may be unaware of the algorithm used by a specific group of cooperative players, and given that the Nash equilibrium is not unique, they may opt to use a single-player bandit algorithm to select arms instead. In this work, we also show that if even one player deviates from the common algorithm, the cooperative players are forced to suffer linear regret $\Omega(T)$. 

Against this background, our work addresses the limitations of existing algorithms by proposing a robust algorithm whose performance is agnostic to the attack strategy. 

\subsection{Contributions}
In this work, we consider two feedback models: distinguishable and  non-distinguishable  collision sensing. In distinguishable sensing (\cite{jammers}), each defender receives feedback on whether a collision occurred and can also distinguish whether the collision occurred from an attacker or a defender. In non-distinguishable sensing (\cite{bistritzLeshem2018,sicgt,beacon2021}), a defender receives feedback on whether a collision occurred and does not have the capability to distinguish whether a collision occurred from an attacker or a defender. 
Our main contributions are the following:
\begin{itemize}
    \item First, we show that the representative existing algorithms in the literature, namely \textsc{mc, sic-mmab, sic-gt, cnj, cuj} and \textsc{cdj}, are not robust to adversarial collisions. More specifically, only $O(\log{T})$ adversarial collisions from a single attacker are sufficient to ensure that the expected regret of these algorithms scales linearly as $\Omega(T)$, where $T$ is the total number of rounds for which the players interacts in the MP-MAB setting. Table \ref{tab:comparisons} summarizes the results showing ``non-robustness" for these existing algorithms.

    \item In the non-distinguishable collision-sensing setting, we propose a novel algorithm \textsc{resync} which exhibits robust behaviour to adversarial collisions. We show that the expected regret of this algorithm deteriorates gracefully as $\tilde{O}(C)$, where $C$ is the total number of adversarial collisions from the attackers. We also shows that this scaling with $C$ is order-optimal up to logarithmic factors in $T$ by proving a corresponding lower bound in MP-MAB setting which scales as $\tilde{\Omega}(C)$. 
    \item In the distinguishable collision-sensing setting, we propose another novel algorithm  \textsc{resync2}, and show that the expected regret of this algorithm also deteriorates linearly as $\tilde{O}(C)$. This scaling of the expected regret is order-optimal in $C$. Due to an additional feedback information available in this setting, the regret bound of \textsc{resync2} is further improved by logarithmic factors in $T$ in comparison to \textsc{resync}. 
    \item Finally, we present several experiments which validate our theoretical findings.
    
\end{itemize}

We defer the proofs and experiments to Appendix \ref{app:proofs} and \ref{app:experiments}, respectively.

\section{Preliminaries}
\label{sec:prelim}

We consider a multi-player variant of the standard stochastic MAB problem with $K$ arms, denoted by set $[K] = \{ 1, \cdots, K \} $. For each arm $k \in [K]$  at time (or round) $t\leq T$, we denote its realized reward   by $X_k(t) \in [0,1]$ drawn i.i.d according to distribution $\nu_k$ with expectation $\mu_k$. Additionally, we assume that the expected rewards of the arms are different, namely $\mu_{(1)} > \mu_{(2)} > \dots > \mu_{(K)}$ where $\mu_{(i)}$ denotes the $i$-th largest expected reward. We denote the number of attackers by $M$ and the number of defenders by $N$ such that $K \geq N$ (as commonly assumed in  \cite{meharabianBoursier2020,bubeckSellke2020,wangProutiere2020,beacon2021,shiAttack2021}). We index the defenders using the set $\{1,\dots,N\}$ and the attackers using the set $\{N+1,\dots N+M\}$. 

In decentralized MP-MAB, at each round $t \leq T$, each player $j\in [N+M]$ pulls an arm $\pi^{j}(t)\in [K]$ and receives a reward
\begin{equation}
    r^{j}(t) = X_{\pi^{j}(t)}(t)(1 - \eta_{\pi^{j}(t)}(t)),
\end{equation}
where $\eta_{k}(t) := \mathds{1}(|1\leq j \leq N+M : \pi^{j}(t) = k\}| > 1)$ is the \emph{collision indicator} (this value is $1$ if more than one player pulls that arm, otherwise remains $0$). At each round $t$, each player who pulled an arm observes their collision indicator $\eta_{\pi^{j}(t)}(t)$ and the corresponding reward $r^{j}(t)$. If a collision occurred, namely $\eta_{\pi^{j}(t)}(t)=1$, then the defenders may or may not receive additional information to distinguish whether that collision occurred with defenders, attackers or both. Based on the availability of this additional information, we consider two feedback models: Non-distinguishable and Distinguishable collision sensing. 

In \emph{non-distinguishable collision sensing,} the feedback is limited to the corresponding reward $r^{j}(t)$ and the collision indicator $\eta_{\pi^{j}(t)}(t)$ at each round $t$ for player $j$. No additional information is available to the defenders to distinguish between the players causing the collisions.

In \emph{distinguishable collision sensing,} the defenders receive the information about the nature of the players who caused the collision. More specifically, at each round $t$, defender $j$ observes the reward $r^{j}(t)$ and the collision indicators

\begin{align}
& \eta^{D}_{k}(t) := \mathds{1}(|\{1 \leq j \leq N : \pi^{j}(t) = k\}| > 1) \label{eq:ic1} \\
& \eta^{A}_{k}(t) := \mathds{1}(|\{N < j \leq N+M : \pi^{j}(t) = k\}| \geq 1) \label{eq:ic2}
\end{align}

where $\eta^{D}_{k}(t)$ and $\eta^{A}_{k}(t)$ represent if the collision occurred due to a defender and an attacker, respectively. Note that $\eta_{\pi^{j}(t)}(t) = \eta^{A}_{\pi^{j}(t)}(t) \lor \eta^{D}_{\pi^{j}(t)}(t)$. These indicators enable the defenders to distinguish between collisions from an attacker and a defender. As an illustration of this feedback model from \cite{jammers}, in cognitive radio networks (CRNs) defenders would be able to distinguish between collisions between themselves and attackers through acknowledgments. At the end of each round, each defender receives an ACK/NACK feedback. If a collision happens due to other defenders but not the attackers, each defender receives a NACK signal. However, if the collision is due to the attackers alone, no NACK signal is received.  Finally if collision occurs simultaneously, a corrupted NACK signal is received. 


The performance of an algorithm is measured in terms of expected regret which is defined as the difference between the
maximal expected reward and the algorithm cumulative reward over $T$ steps, namely
\[R_T := T\sum_{k=1}^{N}{\mu_{(k)}} - \sum_{t=1}^{T}\sum_{j=1}^{N}{\mu_{\pi^{j}(t)} \cdot (1 - \
\eta_{\pi^{j}(t)}(t))}.\]

The maximal expected reward corresponds to the top $N$ arms, also referred to as the \emph{optimal} set of arms.

The goal of the attackers is to force the defenders to incur linear regret while keeping the number of collisions they induce to be as small as possible to remain stealthy. The motivation for stealthy attacks is best illustrated using an example. As noted by \cite{jammersOld2015}, in CRNs, licensed users are protected by law. Therefore users of the network are motivated to reduce their interference with licensed users, since there will be heavy penalties if the licensed users detect prolonged interference from them (\cite{jammersOld2011}). Therefore from an attackers perspective, collisions (or interference in the example) against defenders (or licensed users in the example) are inherently expensive. Hence, we evaluate the defenders' regret in terms of the number of times the attacker collides with the defender. Thus, the attack cost $C$ is the number of adversarial collisions encountered by the defenders, namely
\begin{align}
        C = \sum_{t=1}^{T}\mathds{1}\Big(\exists d , \exists a : 1 \leq d \leq N, N+1 \leq a \leq N+M,\nonumber \pi^{d}(t) = \pi^{a}(t)\Big)
\end{align}
We exclusively consider the case where the attack cost $C$ is unknown to the defenders. In other words, the defenders are \emph{agnostic} to $C$.
{
\section{Limits of Existing Algorithms}
\label{sec:limitsofexistingalgos}
This section shows the limitations of existing algorithms against a single attacker who has no prior information about the defenders and expected rewards of arms. 
 

\textit{Attack on \textsc{mc}}: In \cite{MC}, the ``Musical Chairs" subroutine in the \textsc{mc} algorithm is used to allocate players to optimal arms. This subroutine has inspired many algorithms to follow the same (or slightly modified) procedure (\cite{bessonKaufmann2018,wangProutiere2020,lugosiMeharabian2018,beacon2021}). By successfully attacking this subroutine in \textsc{mc} algorithm, we show that there is a critical threat to any follow-up algorithms that use this subroutine. The following proposition proves that the \textsc{mc} algorithm is not robust to a single attacker.

\begin{proposition}
\label{prop:mcattack}
Assuming that the defenders use \textsc{mc}, there exists an attack strategy with expected attack cost $O\left( \max{\left( {K}\log{(K^2T)}/{\Delta_{min}^2}, K^2\log{T} \right)}\right)$ which ensures that the expected regret of the defenders is $\Omega(T)$, where $\Delta_{min} = \min_i{\mu_{(i)} - \mu_{(i+1)}}$.
\end{proposition}

\textit{Attack on \textsc{sic-mmab}:} The seminal work of \cite{sicmmab} introducing the \textsc{sic-mmab} algorithm with implicit communication through forced collisions has inspired many algorithms to follow suit (\cite{shiECSIC2020,mageshVeeravalli2019b,darak2019,meharabianBoursier2020,huang2021,wangProutiere2020}). 
We present an attack on \textsc{sic-mmab} algorithm which can be slightly modified to show the non-robustness of these follow up works. 
The following proposition proves that the \textsc{sic-mmab} algorithm is not robust to a single attacker. 

\begin{proposition}
\label{prop:sicmmabattack}
Assuming that the defenders use \textsc{sic-mmab}, there exists an attack strategy with expected attack cost $O\left(K^2\log{T}\right)$ which ensures that the expected regret of the defenders is $\Omega(T)$.
\end{proposition}

\textit{Attack on \textsc{sic-gt}:}
The \textsc{sic-gt} algorithm by \cite{sicgt} is robust to selfish players, namely the attacker is a reward maximizing player who gains reward from pulling arms. Playing \textsc{sic-gt} is an $\epsilon$-Nash equilibrium, and this equilibrium is achieved through ``Grim Trigger" (\cite{grimTrigger1971}) punitive strategies. However, the following proposition proves that the \textsc{sic-gt} algorithm is not robust to a single attacker.

\begin{proposition} 
\label{prop:sicgtattack}
Assuming that the defenders use \textsc{sic-gt}, there exists an attack strategy with expected attack cost $O(K^2\log T)$ which ensures that the expected regret of the defenders is $\Omega(T)$.
\end{proposition}
The attack strategies are presented in Appendix \ref{app:proofsforattack}. 

}
\section{The RESYNC Algorithm}
\label{sec:RESYNC}

In this section we consider the non-distinguishable collision sensing model, where defenders cannot determine whether a collision is from an attacker or a defender. We propose the algorithm \textsc{resync} which is robust to adversarial collisions. 
We make the following assumption on the prior knowledge of defenders: 
\begin{assumption}
\label{assm:cmn-knwlg}
All defenders know the number of defenders $N$. Additionally, each defender has a unique identification number $j\in [N]$, also referred as internal rank, and has knowledge of their own internal rank.
\end{assumption} 

The \emph{internal rank}  allows the defenders to coordinate in the bandit game. In Section \ref{sec:RESYNC2} where we consider the distinguishable collision case, we remove this assumption, and use the additional feedback to devise a sub-routine to estimate $N$ and $j$.

\subsection{Description of \textsc{resync}}
\label{sec:RESYNC-desc}
\textsc{resync} consists of an exploration phase and an exploitation phase, where a defender may alternate between phases based on the collision feedback. The algorithm divides the time-horizon $T$ into successive epochs of size $T_B = T_0 + 2N^2 + N$ rounds where $T_0 = 8K\lceil{\log(2K^2T)}/{\Delta^2} \rceil$. Both exploration and exploitation phases run for $T_B$ rounds. At any epoch some defenders may be in the exploration phase while other defenders may be in the exploitation phase. In both of these phases, each defender will choose whether or not to ``restart", 
and enter the exploration phase in the next epoch. Whether or not a defender chooses to restart depends on the collisions sensed by the defender in each phase. We say that the defenders are \emph{synchronized} over an epoch if all of them are in the same phase (exploration/exploitation) in that epoch. Otherwise we say that the defenders are \emph{desynchronized} over that epoch.

{ \textsc{resync} relies on three key ideas to solve the limitations of the existing approaches in the literature. First, the algorithm obtains sufficient number of uncorrupted observations to determine the optimal set of arms with high probability in a decentralized manner. Second, the algorithm maintains synchronization over the defenders against attackers that challenge to desynchronize the system. Third, the algorithm ensures that the defenders pull the optimal set of arms in an orthogonal fashion (with no collisions amongst themselves), for the majority of the time-horizon.}

The outline of \textsc{resync} (\textbf{R}estart \textbf{Syn}chronously under Adversarial \textbf{C}ollisions) is provided in Algorithm \ref{algo:RESYNC}, with the exploration and exploitation protocols provided in Subroutines \ref{sub:exploration} and \ref{sub:exploitation} respectively. We use $t \in [T]$ to denote the round of the bandit game, and assume every defender knows about $t$ at all times.  

\LinesNumbered
\begin{algorithm}
\small
\caption{\textsc{resync}}\label{algo:RESYNC}
\SetKwInOut{Input}{Input}
\KwIn{$T$ (horizon), $N$ (number of defenders), $j$ (internal rank), $T_B$ (length of an epoch)}
\textbf{Initialize:} Restart $\gets $ True, 
Opt $\gets \emptyset$, $\forall i : \tilde{\mu}_i \gets 0, o_i \gets 0, s_i \gets 0 $ \\ 
\For{$\lfloor\frac{T}{T_B}\rfloor$ {\normalfont epochs}}{


\eIf{\normalfont Restart}{
(Restart, Opt) $\gets $ Exploration(Restart, $j, \forall i \; o_i,s_i$) \tcp*{Restart $\gets$ True $\iff$ collision occurs during sequential hopping or verdict of (intra/inter)-communication phase is to Restart}
}{
Restart $\gets $ Exploitation(Restart, $j$, Opt) \tcp*{Restart $\gets$ True $\iff$ collision occurs during inter-communication phase}
}
}
\end{algorithm}


\begin{subroutine}
\small
\caption{Exploitation}\label{sub:exploitation}
\textbf{{Input}:}{Restart, $j$, Opt}; \textbf{Output} Restart \\
\textbf{Initialize} Restart $\gets$ False,  and 
$T_0 \gets 8K\lceil{\log(2K^2T)}/{\Delta^2} \rceil$\\
\For{ $T_0 + 2N^2$ times steps}{
Set $k = t+j$ (mod $N$); Pull Opt$[k]$ 
}

\For(\tcp*[f]{Inter-communication phase}){$N$ rounds}{
Set $k = t+j$ (mod $N$)\\
Pull Opt$[k]$ and receive $\eta_{\text{Opt}[k]}$ \\
\If{$\eta_{\text{Opt}[k]} = 1$}{Restart $\gets$ True}
}

\Return{{\normalfont (Restart)}}
\end{subroutine}


\textbf{Exploration.} (Subroutine \ref{sub:exploration}) Within the exploration phase, there are four sub-phases, which are as follows: sequential hopping that runs for $T_0$ rounds, sensing that runs for $N^2$ rounds, intra-communication that runs for $N^2$ rounds and inter-communication that runs for $N$ rounds.\\

{\textbf{Sequential Hopping.} (Subroutine \ref{sub:exploration} lines 4-16) This sub-phase lasts for $T_0 = 8K\lceil{\log(2K^2T)}/{\Delta^2} \rceil$ rounds. At each round, each defender pulls an arm, namely $j + t$ (mod $K$) $\in [K]$, based on its internal rank $j$ and round of the bandit game $t \in [T]$. If all the defenders are in exploration phase, then this strategy ensures that there is no collision between the defenders. At each round $t$, the defender will also observe the reward and collision indicator. For the rounds when the the collision indicator is $0$ indicating no collision, the defender maintains the cumulative sum of observed rewards and number of observations for each arm $k\in [K]$. If there is a collision, then the \emph{Restart} variable is set to True and the observation received is ignored. If there are sufficient reliable observations for all arms, then the defender sets the flag \emph{SufficientObservations} to be True. }\\

\textbf{Sensing.} (Subroutine \ref{sub:exploration} lines 17-29) This sub-phase lasts for $N^2$ rounds. In this sub-phase, a defender attempts to detect whether there is at least one other defender in the exploitation phase using the collision feedback she receives. This sub-phase is used to help defenders synchronize in the next epoch if they are desynchronized in the current epoch. 
In the sensing sub-phase, if a defender has the flag \emph{SufficientObservations} to be True, then she pulls the optimal arm with least index (which we refer to as Opt$[1]$) in the $jN\text{-th},\dots,(jN+N-1)-\text{-th}$ rounds of this sub-phase, and pulls arm Opt$[1] + 1$ (mod $K$) in the other rounds of the sub-phase. Note that this defender pulls Opt$[1]$ for $N$ consecutive rounds in this sub-phase, and sets \emph{Restart} to True if a collision occurs in those $N$ rounds. \\


\textbf{Intra-Communication.} (Subroutine \ref{sub:exploration} lines 30-39) In this sub-phase each defender in exploration phase communicates with every other defender in the same phase, whether or not to re-enter the exploration phase in the next epoch based on whether every defender has \emph{Restart} to be True. Each defender has her own communicating arm, corresponding to her internal rank. When the defender $i$ is communicating, she sends a bit at a round to the defender $k$ by deciding which arm to pull: a $1$ bit is sent by pulling the communicating arm of defender $k$ (a collision occurs and collision feedback is received by defender $k$) and a $0$ bit is sent by pulling her own arm.

During this phase if a defender has \emph{Restart} to be True, then she sends a $1$ bit to each defender in order to signal a restart. If a defender has \emph{Restart} to be False, then she sends a $0$ bit to every defender. Crucially, an attacker can change this $0$ bit to a $1$ bit by inducing adversarial collisions, however the attacker cannot change a $1$ bit to a $0$ bit, since collisions from one defender to another cannot be reversed by an attacker. Therefore if some defender in the exploration phase has insufficient observations on some arm, every defender restarts, and enters an exploration phase in the next epoch.\\

\textbf{Inter-Communication.} (Subroutine \ref{sub:exploration} lines 40-44) This sub-phase is used for defenders in exploration phase to communicate with the defenders in the exploitation phase within the same epoch so that they can synchronize in the next epoch. In the inter-communication sub-phase, if \emph{SufficientObservations} is True, then the defender pulls Opt$[1]$ for $N$ consecutive rounds to signal to any possible defenders in the exploitation phase, that they should enter the exploration phase in the next epoch.  \\

\textbf{Exploitation.} (Subroutine \ref{sub:exploitation}) In the exploitation phase, the defender sequentially hops over the optimal set of arms throughout the entire epoch. The last $N$ rounds of this protocol is the inter-communication sub-phase within the exploitation phase (Subroutine \ref{sub:exploitation} lines 5-9) used to receive communication from defenders in exploration phase. If a defender in exploitation phase experiences a collision during the inter-communication phase, then this is interpreted as a signal from at least one defender in the exploration phase to enter the exploration in the next epoch.  \\

\begin{subroutine}[h!]
\small
\DontPrintSemicolon
\caption{Exploration}\label{sub:exploration}
\textbf{Input}
{Restart, $j$, and for all $i\in [K],$ $ o_i,s_i$}; \\
\textbf{Output:} {Restart, Opt}\\
\textbf{Initialize:} Restart $\gets$ False,
SufficientObservations $\gets$ False, and
$T_0 \gets 8K\lceil{\log(2K^2T)}/{\Delta^2} \rceil$\\
\For(\tcp*[f]{Sequential hopping phase}){$T_0$ rounds}{
Pull $k = t+j$ (mod $K$) and receive $\eta_k$ and $r_k(t)$ \\
\uIf{$\eta_k = 0$ }{
$o_k \gets o_k + 1$\\
$s_k \gets s_k + r_k(t)$
}
\Else{
Restart $\gets$ True 
}
}

\eIf{$\forall i : \; o_i \geq T_0/K$}{
SufficientObservations $\gets$ True\\
For all $i\in [K]$, we have $ \tilde{\mu}_i = s_i/o_i$\\
Opt $\gets$ List of $N$ best empirically performing arms sorted according to arm index $i\in [K]$.\\
}{
Restart $\gets$ True
}

\For(\tcp*[f]{Sensing phase}){$(i,k) \in [N]\times[N]$ }{
\eIf{$j = i$}{
\eIf{\normalfont SufficientObservations}{
Pull Opt$[1]$ and receive $\eta_{Opt[1]}$ \tcp*{{attempt to sense the presence of defenders  that are in exploitation phase}}
\If{$\eta_{Opt[1]} = 1$}{Restart $\gets$ True}
}{
Pull $1$
}

}{
\eIf{\normalfont SufficientObservations}{
Pull Opt$[1] + 1$ (mod $K$) 
}{
Pull $1$
}
}
}

\For(\tcp*[f]{Intra-communication phase}){$(i,k) \in [N]\times[N]$ }{
\eIf(\tcp*[f]{send}){$j = i$}{
\uIf{\normalfont Restart}{Pull $k$}
\Else{Pull $j$}
}(\tcp*[f]{receive}){
Pull $j$ and receive $\eta_j$\\
\If{$\eta_j = 1$}{Restart $\gets$ True}
}
}

\For(\tcp*[f]{Inter-communication phase}){$N$ rounds}{
\eIf{\normalfont SufficientObservations}{
Pull Opt$[1]$ \tcp*{{to notify any defenders in exploitation phase to rejoin exploration in the next epoch}}
}{
Pull $1$
}
} 

\Return{{\normalfont (Restart, Opt)}}
\end{subroutine}

\newpage
\subsection{Analysis of the \textsc{resync} Algorithm}
Theorem \ref{thm:mainRESYNC} bounds the expected regret incurred by \textsc{resync} given the attack cost $C$. 

\begin{theorem}
\label{thm:mainRESYNC}
Assume $N$ defenders run \textsc{resync} against $M$ attackers, with $\Delta = {\mu_{(N)}-\mu_{(N+1)}}$. Given the attack cost is $C$, the expected regret of \textsc{resync} is bounded by \[O\bigg(CN^2 + CK\frac{\log(K^2T)}{\Delta^2}\bigg)\]
\end{theorem}

The proof of Theorem \ref{thm:mainRESYNC} is composed of three key arguments whose main ideas are presented below. Formal proof appears in Appendix \ref{sec:analysisRESYNC}. 

We begin by upper bounding the number of rounds required for all defenders to collect sufficient observations for each arm, to determine the optimal set of arms with high probability. The following Lemma can be easily derived from~\cite{MC}. 


\begin{lemma}
\label{lem:informalsuffobs}
If all defenders have collected at least $8\lceil{\log(2K^2T)}/{\Delta^2} \rceil$ observations for each arm, then all defenders have determined the optimal set of arms with probability at least $1 - {1}/{T}$. 
\end{lemma}

The concept of restarting under insufficient observations, and the design of the intra-communication phase where all defenders can communicate whether or not they have sufficient observations for each arm (in the presence of attackers) yields the following two ``synchronization" Lemmas:
\begin{lemma}
\label{lem:insuff-renter-explr}
Suppose all defenders are in exploration phase in a certain epoch. If there exists a defender who does not have at least $8\lceil{\log(2K^2T)}/{\Delta^2} \rceil$ observations for each arm, then all defenders re-enter the exploration phase in the next epoch.
\end{lemma}

\begin{lemma}
\label{lem:explore}
All defenders have collected at least $8\lceil{\log(2K^2T)}/{\Delta^2} \rceil$ reliable observations for each arm after at most $T_B + C$ rounds since the start of the game 
\end{lemma}

We then state the third synchronization lemma that completely describes the behaviour of defenders over epochs in the absence of adversarial collisions. 

\begin{lemma}
\label{lem:sync2}
Suppose all defenders are in exploration phase or all defenders are in exploitation phase in a certain epoch. If no attacker collides with any defender during this epoch then all defenders enter exploitation phase in the next epoch.
\end{lemma}

The following is the situation where desynchronization occurs under adversarial collisions. Suppose all defenders are in exploration phase or all defenders are in exploitation phase in a certain epoch. If the attackers cause $N_1$ defenders trigger restart and $N_2$ to not, then the $N_1$ defenders re-enter exploration phase while the $N_2$ enter exploitation phase. What is crucial is that all defenders are able to synchronize in the next epoch, which the final synchronization lemma proves is true. 

\begin{lemma}
\label{lem:sync3}
Suppose $N_1 \geq 1$ defenders are in exploration phase and $N_2$ defenders are in exploitation phase in a certain epoch. Then independent of the arms pulled by attackers during the epoch, all defenders enter exploration phase in the next epoch. 
\end{lemma}

The above Lemma holds due to the design of the sensing sub-phase (where defenders in exploration phase can sense the presence of defenders in the exploitation phase in the presence of attackers) and inter-communication sub-phase (where defenders in exploration phase communicate with defenders in exploitation phase and force them to rejoin exploration in the next epoch).

Finally using the synchronization lemmas, we can upper bound the number of epochs in which not all defenders are in exploitation phase, which is the main argument to bound the regret. To achieve this we consider the transition dynamics of the system of defenders over epochs (see Figure \ref{fig:dynamics}). For the formal description refer Appendix \ref{sec:metagame}.

The states $\textsc{explore},\textsc{desync}$ and \textsc{exploit} in Figure \ref{fig:dynamics} correspond to the state of the bandit game in a certain epoch where all defenders are in exploration phase, at least one defender is in exploration phase and at least one defender is in exploitation phase, or all defenders are in exploitation phase respectively. Similarly the action space $\{\texttt{N},\texttt{C},\texttt{C'}\}$ corresponds to the actions the attackers can take over epochs. Formally, $\texttt{N}$ corresponds to arm pulls during an epoch that cause no collisions with defenders, $\texttt{C}$ to arm pulls during an epoch that cause all defenders to restart exploration on the next epoch, and $\texttt{C'}$ to arm pulls during an epoch that cause $N_1 : 1<N_1<N$ defenders to restart respectively. Here we upper bound the number of epochs the state is not \textsc{exploit}, given that the number of epochs in which actions $\texttt{C} \text{ and }\texttt{C'}$ are played by the attackers are at most $C$.

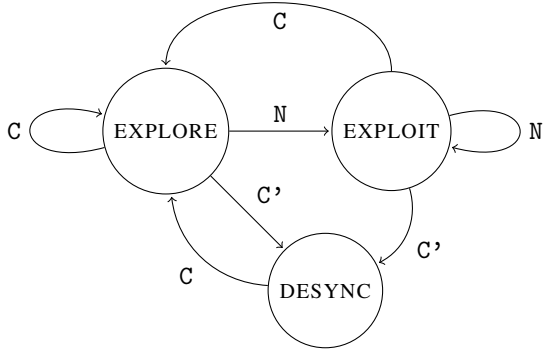
\begin{figure}
\centering
\begin{tikzpicture}[shorten >=1pt,node distance=3cm,on grid,auto] 
   \node[state] (1)   {\textsc{explore}}; 
   \node[state] (2) [below right=of 1] {\textsc{desync}}; 
   \node[state](3) [right=of 1] {\textsc{exploit}};
    \path[->] 
    (1) edge [loop left] node {\texttt{C}}  (1)
    (1) edge  node {\texttt{C'}}  (2)
    (1) edge  node {\texttt{N}}  (3)
    (2) edge  [bend left = 40] node {\texttt{C}}  (1)
    (3) edge [bend right = 90] node {\texttt{C}} (1)
    (3) edge [bend left = 40] node {\texttt{C'}} (2)
    (3) edge [loop right] node {\texttt{N}} (3)
    ;
\end{tikzpicture}
\caption{Transition Dynamics of the System of Defenders over epochs}\label{fig:dynamics}
\end{figure}

\begin{lemma}
\label{lem:goodstatebound}
(Informal) After $\lceil1 + {C}/{T_B}\rceil$ epochs since the start of the game, the number of epochs in which, not all defenders are in exploitation phase, is at most $O(C)$. 
\end{lemma}
When all defenders are in exploitation phase, they pull the top $N$ arms with no collisions amongst themselves, conditioned on the event that all of them have determined the optimal set of arms. These arguments show that the leading term in the expected regret is $O(CT_B) = O(CN^2 + CK{\log(K^2T)}/{\Delta^2})$.  

\section{The RESYNC2 Algorithm}
\label{sec:RESYNC2}

This section proposes \textsc{resync2} for the distinguishable collision sensing setting, where defenders can determine whether a collision is from an attacker or a defender. We remove Assumption \ref{assm:cmn-knwlg} and show that the additional feedback allows  \textsc{resync2} to have better performance than \textsc{resync}.

The algorithm \textsc{resync2}, consists of three phases that run sequentially, namely, initialization, exploration and exploitation. Note that unlike \textsc{resync}, the exploration and exploitation phases are not intertwined, thereby removing the necessity of the sensing and inter-communication sub-phases in \textsc{resync}.

\textbf{Initialization.} The purpose of the initialization phase is to estimate the total number of defenders, and assign distinct internal ranks between the defenders. The initialization phase is similar to the one from \textsc{sic-mmab} (\cite{sicmmab}) although adapted to the distinguishable collision sensing setting, so number of defenders and internal ranks can be determined in the presence of attackers.

\textbf{Exploration.} The exploration phase progresses in epochs of size $2K$, with two sub-phases, sequential hopping (from \textsc{resync}) but which lasts only for $K$ rounds and a modified intra-communication phase (also from \textsc{resync}) that also lasts only for $K$ rounds. The interplay between sequential hopping and inter-communication is similar to that of \textsc{resync} in the sense that, the outcome of sequential hopping (whether one reliable observation was received for each arm during sequential hopping) is communicated through forced collisions in the intra-communication phase. We say an epoch is successful if each defender received one reliable observation for each arm during sequential hopping. The defenders explore until there are $8\lceil{\log(2K^2T)}/{\Delta^2} \rceil$ successful epochs.

\textbf{Exploitation} In the exploitation phase, the defenders sequentially hop over the optimal set of arms until the end of the time-horizon.

The complete pseudocode of \textsc{resync2} is given in Appendix \ref{app:RESYNC2-desc}.

The performance of \textsc{resync2} improves from that of \textsc{resync} due to the additional feedback available in distinguishable collision sensing. This feedback allows defenders to communicate robustly using collisions and collision feedback even in the presence of attackers. That is, they can determine with certainty whether a $1$ bit or a $0$ bit was sent by a defender during an inter-communication phase, by reading the collision indicator $\eta^{D}_k$. In the previous feedback setting, only $\eta_k$ is available and the bits sent in this manner between defenders can be modified by an attacker through adversarial collisions. This robust communication allows us to disentangle the exploration and exploitation phases from \textsc{resync}, leading to the size of the epochs to reduce from $O\bigg(N^2 + K{\log(K^2T)}/{\Delta^2}\bigg)$ to just $O(K)$, thereby improving the regret guarantee.

\subsection{Analysis of the \textsc{resync2} Algorithm}
Theorem \ref{thm:RESYNC2} bounds the expected regret incurred by \textsc{resync2} given the attack cost $C$. 
\begin{theorem} 
\label{thm:RESYNC2}
Assume $N$ defenders run \textsc{resync2} against $M$ attackers, with $\Delta = \mu_{(N)}-\mu_{(N+1)}$. Conditioned on the number of attacks being $C$, the expected regret of \textsc{resync2} is bounded by \[O\bigg(CK + K\frac{\log(K^2T)}{\Delta^2}\bigg)\]
\end{theorem}


The proof of Theorem \ref{thm:RESYNC2} is composed of two key arguments whose main idea is presented below, formal proof appears in Appendix \ref{sec:analysisRESYNC2}. 

We begin by showing that after initialization phase is complete, all defenders have estimated the number of defenders $N$ correctly, and have distinct internal ranks in $[N]$ with high probability. It is worth noting that the initialization phase is successful with high probability, irrespective of the arms pulled by the attackers, whereas if the same initialization phase from \cite{sicmmab} was used, this would not be the case. 

Then similar to the analysis for Theorem \ref{thm:mainRESYNC}, we upper bound the number of rounds required for all defenders to collect sufficient observations for each arm, in order to determine the optimal set of arms with high probability using Lemma \ref{lem:informalsuffobs} and the following Lemma.

\begin{lemma}
\label{lem:explore-dc}
After at most $C + 8{\log{(2K^2T)}}/{\Delta^2}$ epochs of exploration, all defenders have collected at least $8\lceil{\log(2K^2T)}/{\Delta^2} \rceil$ observations for each arm.   
\end{lemma}

After each defender has sufficient observations on each arm, all defenders will enter the exploitation phase and never leave this phase, where they will pull the top $N$ arms orthogonally until the end of the time-horizon. These arguments show that the leading term in the expected regret is $O\bigg(CK + K{\log(K^2T)}/{\Delta^2}\bigg)$. 
\section{Lower Bounds}
\label{sec:lowerbounds}
In this section, we consider lower bounds on expected regret of the defenders, when the attack cost is $C$, and will show that the expected regret of the defenders has a linear dependence on the attack cost. 

The following theorem establishes a lower bound on the expected regret of any MP-MAB algorithm in terms of $C$.
\begin{theorem}
\label{thm:lowerbound1}
There exists an attacker with expected number of attacks at most $C$, for which any MP-MAB algorithm suffers expected regret $\Omega({NC}/{K})$. 
\end{theorem}
\begin{proof}
The attacker samples an arm $k \sim \mathcal{U}(K)$, pulls $k$ for $C$ rounds, and pulls no arm for the remaining rounds. Clearly the number of collisions any defender will face from the attacker is at most $C$. The sampled arm is in the optimal set of arms with probability $\frac{N}{K}$. Under this event, during the first $C$ rounds, any algorithm has per-round regret at least $\mu_{(N)}-\mu_{(N+1)}$. So, the expected regret over $T$ rounds under this attack is at least $C\cdot(\mu_{(N)}-\mu_{(N+1)})\cdot\frac{N}{K} = \Omega\left(\frac{NC}{K}\right)$.
\end{proof}

Hence, this lower bound establishes that both \textsc{resync} and \textsc{resync2} exhibit  order-optimal behaviour in terms of $C$.

\section{Discussion}
We studied a setting in which $N$ defenders collaborate to minimize regret from a multi-armed bandit where several players simultaneously pull arms and $M$ attackers disrupt collaboration between defenders. We showed that even when $M = 1$, existing algorithms, including algorithms robust to jammers and selfish players, incur linear regret with only logarithmic number of collisions from the attacker. We thus proposed the algorithm \textsc{resync} and \textsc{resync2} based on restarting synchronously under adversarial collisions in which the performance deteriorates gracefully as the number of collisions from multiple attackers increases. We then provided lower bound that proves that the regret scales linearly with the number of collisions from attackers. In conclusion, we establish that our proposed algorithms are order-optimal in terms of the attack cost.

This work leaves several questions open. Firstly, although the assumption that a lower bound on $\Delta$ is known can be removed in the distinguishable collision sensing setting (refer Appendix \ref{app:removedelta}), by combining our synchronization mechanism with a generalization of the Successive-Eliminations algorithm (\cite{perchet2013}) as in \textsc{sic-mmab} \cite{sicmmab}, it is unclear what algorithmic mechanism can be used in the non-distinguishable collision sensing setting when $\Delta$ is unknown and attackers exist in the game. Next, it would be interesting to remove Assumption \ref{assm:cmn-knwlg} in the non-distinguishable collision sensing setting, either with algorithms that robustly estimate the value $N$ online or by using an approach that does not require knowledge of $N$ (as in \cite{bistritzLeshem2018}). Further, one may look at robustness to adversarial collisions in the no-sensing setting where no collision information is observed and the heterogeneous setting where the arm means vary among players.

\printbibliography
\newpage
\appendix

\section{Proofs}
\label{app:proofs}
We use the following standard concentration bound. Let $X_1,\dots,X_t$ be independent variables such that $X_i \in [0,1]$ almost surely. Consider the mean $\Bar{X} = \frac{\sum_{i=1}^{t}X_i}{t}$ and its expectation $\mu = \mathbb{E}[\Bar{X}]$.

\begin{lemma}
(Hoeffding's Inequality) For any $\epsilon > 0$,
\[\mathbb{P}\left( |\Bar{X} - \mu| > \epsilon \right) \leq 2\cdot\exp{(-2t\epsilon^2)}\]
\end{lemma}

\begin{definition}
An $\epsilon$-correct ranking of $K$ arms is a sorted list of empirical mean rewards of arms such that $\forall i,j : \tilde{\mu}_i$ is listed before $\tilde{\mu}_j$ if $\tilde{\mu}_i - \tilde{\mu}_j > \epsilon$ 
\end{definition}

We will use $\Delta$ to denote the gap between the expected reward of the $N$-th best arm and $(N+1)$-th best arm.
\begin{lemma}
\label{lem:sortedlist}
If $\epsilon < \Delta$, then in an $\epsilon$-correct ranking, the set indices of the first $N$ arms equals the set of indices of the optimal set of arms. 
\end{lemma}
\begin{proof}
Without loss of generality, let the true means be ranked according to index, that is, $\mu_1 > \mu_2 > \dots > \mu_K$. Since $\epsilon < \Delta$, by definition, $\forall i\leq N : \epsilon < \mu_{N} - \mu_{N+1} \leq \mu_{i} - \mu_{N+1}$ which implies that $\tilde{\mu}_i$ is listed before $\tilde{\mu}_{N+1}$ for all $i\leq N$, in an $\epsilon$-correct ranking.  Similarly $\forall i > N : \epsilon < \mu_{N} - \mu_{N+1} \leq \mu_{N} - \mu_{i}$ which implies that $\tilde{\mu}_{N}$ is listed before $\tilde{\mu}_{i}$ for all $i > N$, in an $\epsilon$-correct ranking. 
\end{proof}

\begin{lemma}
\label{lem:suffobs}
If all defenders have collected at least $8\lceil\frac{\log(2K^2T)}{\Delta^2} \rceil$ observations for each arm, then all defenders have a $\frac{\Delta}{2}$-correct ranking with probability at least $1 - \frac{1}{T}$.
\end{lemma}
\begin{proof}
If for player $j$, $\forall i : |\mu_i- \tilde{\mu}_i| < \frac{\epsilon}{2}$ then player $j$ has an $\epsilon$-correct ranking. Let $\epsilon = \frac{\Delta}{2}$, and $O = 8\lceil\frac{\log(2K^2T)}{\Delta^2} \rceil$.
\begin{align*}
& \mathbb{P}\left( \text{player j does not have an } \epsilon-\text{correct ranking}|\text{ player j viewed}  \geq O \text{ observations}\right)\\
& \leq \mathbb{P}\left( \exists i \in [K] : |\mu_i - \tilde{\mu}_i| > \frac{\epsilon}{2} |\text{ player j viewed}  \geq O \text{ observations}\right)\\
& \leq \sum_{i=1}{K}\mathbb{P}\left(|\mu_i - \tilde{\mu}_i| > \frac{\epsilon}{2} |\text{ player j viewed}  \geq O \text{ observations}\right)\\
& \leq \sum_{i=1}{K}\sum_{v=1}^{\infty}\mathbb{P}\left(|\mu_i - \tilde{\mu}_i| > \frac{\epsilon}{2} |\text{ \#views } = v\right)\mathbb{P}\left(\text{\#views} = v | v \geq O \right)\\
& \leq \sum_{i=1}{K}\sum_{v=1}^{\infty}\left( 2\exp{(\frac{-O\epsilon^2}{2})}\right)\mathbb{P}\left(\text{\#views} = v | v \geq O \right) \leq K\left( 2\exp{(\frac{-O\epsilon^2}{2})}\right)\sum_{v=1}^{\infty}\mathbb{P}\left(\text{\#views} = v | v \geq O \right)\\
& = 2K\exp{\left(\frac{-O\epsilon^2}{2}\right)}
\end{align*}

and so, 
\begin{align*}
& \mathbb{P}\left( \exists j \in [N] : \text{player j does not have an } \epsilon-\text{correct ranking}|\text{ player j viewed}  \geq O \text{ observations}\right) \\
& \leq 2nK\exp{\left(\frac{-O\epsilon^2}{2}\right)} \leq 2nK\exp{\left(-8\left\lceil\frac{\log(2K^2T)}{\Delta^2} \right\rceil\cdot\ \frac{\Delta_{min^2}}{8}\right)} \leq \frac{2nK}{2K^2T} \leq \frac{1}{T} 
\end{align*}

\end{proof}

\begin{replemma}{lem:informalsuffobs}
If all defenders have collected at least $8\lceil\frac{\log(2K^2T)}{\Delta^2} \rceil$ observations for each arm, then all defenders have determined the optimal set of arms with probability at least $1 - \frac{1}{T}$. 
\end{replemma}
\begin{proof}
Immediate consequence of Lemma \ref{lem:sortedlist} and Lemma \ref{lem:suffobs}.
\end{proof}

\subsection{Proofs for attacking current Algorithms}
\label{app:proofsforattack}
For the remainder, we will assume that the attacker can pull $\bot$, which is equivalent to pulling no arm. We will not assume that the defenders have this capability.
\subsubsection{Attack on \textsc{mc}}
We consider \textsc{mc} with $T_0 = \left\lceil \max{\left( \frac{64K}{\Delta_{min}^2}\log{(4K^2T)}, \frac{K^2\log{4T}}{0.02} \right)}\right\rceil$ and $\Delta_{\min} = \min_i{\mu_{(i)} - \mu_{(i+1)}}$, as we consider the setting where the number of attackers are unknown. In the \textsc{mc} algorithm \cite{MC}, there are two phases. In the first phase all defenders estimate the mean of all arms while simultaneously estimating the number of players. In the second phase, all players commit to some arm. The \textsc{mc}-attacker ensures that no defender commits to the arm with the highest mean, by committing to this arm first with high probability. The following is the pseudocode for the \textsc{mc} attack. \\

\begin{algorithm}[H]
\caption{\textsc{mc}-Attack}
\SetKwInOut{Input}{Input}
\KwIn{$T$ (horizon)}
Restart $\gets $ True\\
$C_{T_0} \gets 0,$ Opt $\gets \emptyset, \tilde{\mu}_i \gets 0, o_i \gets 0, s_i \gets 0 $\\

\For{$ 0 \leq t \leq T_0$}{
Pull $k \sim \mathcal{U}(K)$ and receive $\eta_k$ and $r_k(t)$\\
\If{$\eta_k = 0$}{
$o_k \gets o_k + 1$\\
$s_k \gets s_k + r_k(t)$
}
}
$\forall i : \tilde{\mu}_i = s_i/o_i$\\
opt $\gets$ index of best empirically performing arm\\
subopt $\gets$ index of worst empirically performing arm\\
\For{$\lceil K\log{T} \rceil$ rounds}{
Pull opt
}
Pull $\bot$ for the remaining rounds. 

\end{algorithm}

The following is the Proof of Proposition \ref{prop:mcattack} that claims that under the \textsc{mc} attackers, the expected regret of the defenders is $\Omega(T)$.
\begin{proof}[Proof of Proposition \ref{prop:mcattack}]
By Lemma 1 from \cite{MC}, after the first $T_0$ rounds, all defenders have a $\frac{\Delta}{2}$- correct ranking of the arms with probability at least $1-\frac{1}{T}$. Conditioned on this event, by Lemma \ref{lem:sortedlist} the \textsc{mc} attacker has found the arm with highest true mean which we call opt. Then, the \textsc{mc} attacker pulls opt $\lceil K\log{T} \rceil$ rounds. The probability that all defenders have committed to an arm other than opt within $\lceil K\log{T} \rceil$ rounds is at least $1 - \frac{N}{T}$. Under this event the regret during the remaining $T - T_0 - K\log{T}$ rounds is at least $(T - T_0 - K\log{T})(\mu_{(1)} - \mu_{(N+1)})$, whereas the probability this event occurs is at least $(1-\frac{1}{T})(1 - \frac{N}{T})$. Therefore the expected regret under the \textsc{mc} attacker is lower bounded as 
\begin{align*}
    \mathbb{E}R_T &\geq (T - T_0 - K\log{T})(\mu_{(1)} - \mu_{(N+1)})(1 - \frac{N+1}{T})\\
    & T(\mu_{(1)} - \mu_{(N+1)}) - \mu_{(1)} - \mu_{(N+1)}(T_0 + K\log{T} + N + 1) \geq \Omega(T)
\end{align*}
Further the number of attacks is at most $T_0 + K\log{T} = O(T_0)$.
\end{proof}

For a different attack strategy against \textsc{mc} with $T_0$ dependent on the gap between the $(m+n)$-th best arm and $(m+N+1)$-th best arm as assumed by \cite{MC}, refer Section \ref{sec:attackjammers} for a strategy that would work against \textsc{mc} as well.

\subsubsection{Attack on \textsc{sic-mmab}}
\label{sec:attacksicmmab}
We first describe the attack in detail along with pseudocode, then we provide the Proof of Proposition \ref{prop:sicmmabattack}. In this attack, the attacker first picks an arm uniformly at random, which we may call the target arm. The attacker targets the defender with the highest internal rank which we denote as $I$ ($I$ will equal the total number of players $N+M$), whom we may call the target defender, and forces this defender to commit to the target arm with high probability. The attacker collides with the target defender for every round of the exploration phases, in order corrupt the rewards observed by the defender. During communication phases all the information sent from other defenders to the target defender is corrupted by the attacker, via arm collisions. During the part of the communication phase, where the attacker must send arm statistics to the target defender, she sends arm statistics on the target arm to be the highest and all other arms to be the lowest. The precise description for the attack and pseudocode follows. 
During the initialization phase the attacker does not follow the musical chairs protocol, but instead for $T_0 = \lceil K\log{T}\rceil$ rounds pulls only arm $1$. (This is done so that the attacker obtains the lowest internal rank. The idea of low internal rank is important, because, in \textsc{sic-mmab} players with highest internal rank commit to arms first. By obtaining low internal rank, the attacker ensures that the defender commits to an arm first). For the final $2K$ rounds of Initialization the attacker plays the Estimate\_M protocol from \textsc{sic-mmab} used to estimate the total number of players (Note that we denote the total number of players as $I = N+M$).
During the $p$-th exploration phase, when it is time for the defender to explore arms by sequential hopping, the attacker collides with every arm the highest ranked defender picks in sequential hopping. (This can be done since the sequential hopping is a deterministic procedure).

The complete pseudocode for the strategy is provided. \\
\begin{algorithm}[H]
\caption{\textsc{sic-mmab}-Attack}
\SetKwInOut{Input}{Input}
\KwIn{$T$ (horizon)}
$\hat{k} \sim \mathcal{U}(K)$\\
$\pi \gets 0$\\
\For{$\lceil K\log{T} \rceil$ rounds}{
Pull arm $1$
}
$(I,j) \gets $ \text{Estimate\_M}($1$)\\
\For{$18I^2\log{T}$ rounds}{
$\pi \gets I$\\
\For{$K2^p$ rounds}
{
Set $\pi \gets \pi + 1$ (mod $K$) and Pull $\pi$\\
\textsc{sic-mmab}-Adversarial-Communication($p,j,\hat{k}$)
}
}
Pull $\bot$ for the remaining rounds.
\end{algorithm}

During the $p$-th communication phase, when it is time for the attacker to send arm statistics to the highest ranked defender, the attacker sends the observed arm statistic as zero whenever she is communicating rewards from a non-target action and observed arm statistic as $2^{p+1}-1$ on the target action. To all other defenders the attacker sends the observed arm statistic as $2^{p+1}-1$ for all actions. Further, the attacker intercepts all communication that occurs between any other two defenders, by colliding on the receiving players arm. This makes all defenders observe the communication $2^{p+1}-1$ for all actions from all defenders. The complete pseudocode for the adversarial-communication is provided.\\ 

\begin{subroutine}[H]
\caption{\textsc{sic-mmab}-Adversarial-Communication}
\SetKwInOut{Input}{Input}
\KwIn{$p$ (phase number), $j$ (internal rank), $\hat{k}$ (target-arm)}
$\pi \gets 0$\\
Define $E_1 \gets \{(i,l,k) \in [1]\times[I-1]\times[K] | i\neq l\}$\\
Define $E_2 \gets \{(i,l,k) \in [1]\times\{I\}\times[K] | i\neq l\}$\\
Define $E_3 \gets \{(i,l,k) \in ([I]\backslash[1])\times[I]\times[K] | i\neq l\}$\\

\For{$(i,l,k) \in E_1$}{
Send$(l,2^{p+1} - 1,p,j,[K])$\\
}
\For{$(i,l,k) \in E_2$}{
\uIf{$k = \hat{k}$}{Send$(l,2^{p+1} - 1,p,j,[K])$ }
\Else{
Send$(l,0,p,j,[K])$ 
}
}
\For{$(i,l,k) \in E_3$}{
\For{$p+1$ rounds}{Pull $l$} 
}
\end{subroutine}

For completeness we provide the send-protocol from \textsc{sic-mmab}.\\
\begin{subroutine}[H]
\caption{Send}
\SetKwInOut{Input}{Input}
\KwIn{$l$ (player receiving), $s$ (statistic to send), $p$ (phase number) , $j$ (own internal rank), $[K_p]$ (set of active arms)}
$\textbf{m} \gets $ binary writing of $s$ of length $p+1$, i.e, $s = \sum_{b=0}^{p}{m_b2^b}$\\
\For{$b = 0,\dots,p$}{
\eIf{$m_b = 1$}{Pull $l$}{Pull $j$}
}
\end{subroutine}

\begin{proof}[Proof of Proposition \ref{prop:sicmmabattack}]

The attacker first selects the target arm $\hat{k} \sim \mathcal{U}(K)$. The first part of the attack during the initialization phase is successful if, during the end of the initialization phase, both the attacker and all defenders have determined the total number of players to be $M$, and all the players have determined that the internal rank of the attacker is $1$. This happens whenever Musical chairs protocol is successful when attacked by the above strategy where the attacker picks only arm $1$ for $T_0 = \lceil K\log{T}\rceil$ rounds. Lemma 2 from \cite{sicmmab} gives the probability of reaching an orthogonal setting at time $T_0$, to be at least $1 - I\exp\frac{T_0}{K}$. Since each defender picks actions uniformly at random during this phase, the failure probability for a particular defender $j$ is bounded by:

\[\mathbb{P}[\forall t \leq T_0 : \eta^{j}(t)=1] = (1-\frac{1}{K})^{T_0} \leq \exp{(\frac{T_0}{K})} = \exp{-(\frac{\lceil K\log{T}\rceil}{K})} \leq \frac{1}{T} \]

With a union bound over all defenders, the probability initialization succeeds is at least $1-\frac{I}{T}$. We will condition on the success of previous event. The Estimate-M protocol is run with the attacker at internal rank 1 (the attacker follows the subroutine specified in the description of \textsc{sic-mmab}), and all players determine the total number of players to be $I$, and all the players have determined that the internal rank of the attacker is $1$. 

After initialization is complete, because of the attack during exploration phases, the defender, observes a reward of $0$ on every arm, for every single phase. At any phase after communication is over, the target defender, estimates the means of arms as follows. 
\[\tilde{\mu}_{k}^{I}=
\begin{cases}
\frac{2^{p+1}\cdot(I-1) + 0}{2^{p+1}I} = 1-\frac{1}{I}, \text{if } k = \hat{k}\\
\frac{2^{p+1}\cdot(I-2) + 0}{2^{p+1}I} = 1-\frac{2}{I}, \text{if } k \neq \hat{k}
\end{cases}\]

Consider the accept-reject criterion between two arms $i,j \in [K]$:

\begin{align}
    \label{eq:acceptreject} \tilde{\mu}_{i}^{I}-b_s \geq \tilde{\mu}_{j}^{I}+b_s
\end{align}
where $\tilde{\mu_i}$ is the empirical mean of quantized sum of rewards recorded by defender $I$, and $b_s = 3\sqrt{\frac{\log(T)}{2s}}$. If $i,j$ are non-target actions then this condition can never hold since $\tilde{\mu}_{i}^{I} = \tilde{\mu}_{j}^{I}$. Whereas if $i$ is the target action and $j$ is a non-target action then,
$\tilde{\mu}_{i}^{I}-\tilde{\mu}_{j}^{I} = \frac{1}{I}$. Which implies 
\begin{align}
\label{eq:targetnontarget}
    \forall j \neq \hat{k} : \tilde{\mu}_{\hat{k}}^{I}-b_s \geq \tilde{\mu}_{j}^{I}+b_s \text{ holds after at most } s \geq 18I^2\log{T} \text{ rounds}
\end{align}
Analogously, it can be shown that condition (\ref{eq:acceptreject}) between two arms does not hold for ANY defender before $18I^2\log{T}$ rounds. This ensures that the size of the active set of arms and active set of players for all non-target defenders, before $18I^2\log{T}$ time steps is $I$ and $K$ (since all defenders communicate through active set of arms and only communicate to active players, the above guarantees that the defenders do not ``de-synchronize"). It remains to prove that the target defender commits to the target arm. From line 18 of the \textsc{sic-mmab} algorithm \cite{sicmmab}, the size of the accepted-set is at most the number of active players in any phase, which here is at most I. \\

For the target defender, by equation \ref{eq:targetnontarget}, the size of the accept-set at some round before the $18I^2\log{T}$-th round will be exactly $1$ and only the target-arm will be in this set. The target defender accepts this target action since her internal rank is the highest (this is a feature of \textsc{sic-mmab}). Let $\mathcal{E}$ denote the event that $\hat{k}$ is a sub-optimal arm and initialization is successful. Clearly $\mathbb{P}(\mathcal{E}) \geq (1 - \frac{N}{K})\cdot(1 - \frac{I}{T})$. Therefore the expected regret of the defenders under attack is lower bounded by,
\begin{align*}
& \mathbb{E}R_T =  \mathbb{E}\left(R_T | \mathcal{E}\right)\mathbb{P}(\mathcal{E}) + \mathbb{E}\left(R_T | \mathcal{E}^{\mathsf{c}}\right)\mathbb{P}(\mathcal{E}^{\mathsf{c}})\\
& \geq (T-18I^2\log{T} - 2K - K\log{T})\cdot(\mu_{(N)}-\mu_{(N+1)})\cdot(1 - \frac{N}{K})\cdot(1-\frac{I}{T}) + 0 = \Omega\left(\frac{K-N}{K}T\right)
\end{align*}

Since the attacker only pulls arms for at most $18^2\log{T} + 2K + K\log{T} \leq 21K^2\log{T} $ rounds, the number of attacks is at most $O\left( K^2\log{T} \right)$.

\end{proof}

For another attack against the \textsc{sic-mmab} algorithm, we show that the player estimation protocol is not robust. Here the attacker fools at least one defender into estimating the total number of players to be $K+1$. This is a problem since these defenders would attempt to pull arm $K+1$ during the communication phase of the algorithm. Clearly such algorithms cannot be used when attackers exist in the bandit game.

\begin{algorithm}[H]
\caption{\textsc{sic-mmab}-Desynchronization-Attack}
\SetKwInOut{Input}{Input}
\KwIn{$T$ (horizon)}
defenderFound $\gets$ False\\
$\pi \gets 0$\\
\For{$\lceil K\log{T} \rceil$ rounds}{
Pull $\bot$
}
\While{{\normalfont not defenderFound}}{
Set $\pi \gets \pi + 1$ and Pull $\pi$\\
\If{$\eta_{\pi} = 1$}{
defenderFound $\gets$ True
}
}
Set $d \gets \pi$\\
\For{$d$ rounds}{
Pull $\pi$}
\For{$2K-d$ rounds}{
$\pi \gets \pi + 1$\\
Pull $\pi$
Pull $\bot$ for the remaining rounds.
}

\end{algorithm}
\begin{proposition} 
\label{prop:sicmmabdesyncattack}
Assume that the defenders use \textsc{sic-mmab}. Then, the \textsc{sic-mmab} attacker forces some defender to estimate the total number of players to be at least $K+1$ with number of attacks $O(K)$. 
\end{proposition}
\begin{proof}
After the first $K\log{T}$ rounds, orthogonalization is successful with probability at least $1- \frac{N}{T}$. The Estimate-I protocol of \textsc{sic-mmab} runs for $2K$ rounds. During these $2K$ rounds, the attacker hops arms sequentially until he finds the first defender. After this for the remaining rounds the attacker collides with the defender until the end of the Estimate-M protocol. Since the attacker finds a defender within the first $K$ rounds, and collides until the end of the protocol the total number of collisions witnessed by this defender is at least $K+1$. Since the defenders estimate players based on number of collisions, the defender estimates the total number of players to be at least $K+1$. No arm is pulled during the first $K\log{T}$ and at most $2K$ collisions can occur in the $2K$ rounds the attacker is active. The number of attacks is hence $O(K)$. 
\end{proof}

\newpage

\subsubsection{Attack on \textsc{sic-gt}}

In the case that all defenders play \textsc{sic-gt}, if an attacker attempts gain advantage of the system, either by colliding during exploration phase or communication phase, this is detected through the collision feedback by at least one defender who reports this to all other defenders through deliberate collisions. Thereafter all defenders trigger punishment, and pulls arms in a way that punishes the attacker by reducing the reward gained by the attacker to less than what would be possible through cooperation. Unfortunately, by punishing the attacker the defenders incur linear regret, therefore an attacker who seeks to break the system would immediately trigger punishment. This is essence is the strategy for our attacker. 

\begin{algorithm}[H]
\caption{\textsc{sic-gt}-Attack}
\SetKwInOut{Input}{Input}
\KwIn{$T$ (horizon)}

$M,j \gets$ Initialize$(T,K)$\\
Set $l \sim \mathcal{U}([M]\backslash\{j\})$\\
Set $k \gets l + t$ and Pull $k$\\
Pull $\bot$ for the remaining rounds

\end{algorithm}

In \textsc{sic-gt} if punishment is triggered by all defenders then we argue that for remainder of the time-horizon, the defenders will incur linear regret. The first part of the punishment phase consists of each defender $j$ independently estimating the mean $\mu_k$ of arm $k$ with the empirical estimate of defender $j$ denoted by $\hat{\mu}^j_k$. The second part consists of using these estimates in the following way. Each defender $j$ selects actions from a constant probability distribution $p^j$ for the remainder of the time-horizon, with the probability that arm $k$ is picked by defender $j$ in those rounds given by, 
\[p^j_k := 1 - \bigg((1-\frac{1}{K})^{N}\frac{\sum_{l=1}^{N+1}\hat{\mu}^j_l}{(N+1)\hat{\mu}^j_k}\bigg)^{\frac{1}{N}}\]

In Lemma \ref{lem:sicgtpunish} we give the lower bound on the probability that no defender has selected the Dirac measure with probability $1$ on the best arm as her constant probability distribution. In essence this will imply that during punishment phase, the probability that some defender picks the best arm in a particular round is upper bounded by a constant less than $1$, which will lead to linear regret during the second part of the punishment phase.
\begin{lemma}
\label{lem:sicgtpunish}
If the \texttt{PunishHomogeneous} protocol is started at time $T_{\text{punish}}$, then after at most $T_{\text{punish}} + t_p$ rounds, each defender $j$ samples actions from the probability distribution ${p}^j$ such that 
\begin{align*}
&    \mathbb{P}\left({\forall \; j : p^j_{(1)} \leq 1 - \bigg(\frac{2(1-\frac{1}{K})^{(2N)}}{(N+1)(1 + (1-\frac{1}{K})^{N})}\bigg)^{\frac{1}{N}}}\right) \geq 1-\frac{3}{T}
\end{align*}

for the remainder of the time-horizon, where $t_p = O(\frac{K}{(1-\tilde{\alpha})^2\mu_{(K)}}\log{(T)})$ with $\tilde{\alpha} = \frac{1 + (1-\frac{1}{K})^{N}}{2}$.
\end{lemma}
\begin{proof}
The punishment protocol starts for all defenders at $T_{\text{punish}}$. By Lemma 22 from \cite{sicgt}, for $\gamma = (1-\frac{1}{K})^{N}$ and $\delta = \frac{1-\gamma}{1+3\gamma}$, with probability at least $1-\frac{3}{T}$, after round $T_{\text{punish}} + t_p$,
\begin{align}
\label{eq:hoeffdingsicgt}
\forall j \in [N] \; \forall k \in [K] : (1-\delta)\hat{\mu}^j_k < \mu^j_k < (1+\delta)\hat{\mu}^j_k 
\end{align}

Since for each defender $j$, \[p^j_k := 1 - \bigg(\gamma\frac{\sum_{l=1}^{N+1}\hat{\mu}^j_l}{(N+1)\hat{\mu}^j_k}\bigg)^{\frac{1}{N}}\] the bound in equation \ref{eq:hoeffdingsicgt} can be inverted to give, for all defenders $j$,

\begin{align*}
p^j_{(1)} & = 1 - \bigg(\gamma\frac{\sum_{l=1}^{N+1}\hat{\mu}^j_{(l)}}{(N+1)\hat{\mu}^j_{(1)}}\bigg)^{\frac{1}{N}} 
\leq 1 - \bigg(\gamma\frac{1-\delta}{1+\delta}\frac{\sum_{l=1}^{N+1}{\mu}_{(l)}}{(N+1){\mu}_{(1)}}\bigg)^{\frac{1}{N}}\\
& = 1 - \bigg(\frac{2\gamma^2}{1+\gamma}\frac{\sum_{l=1}^{N+1}{\mu}_{(l)}}{(N+1){\mu}_{(1)}}\bigg)^{\frac{1}{N}}
\leq 1 - \bigg(\frac{2\gamma^2}{1+\gamma}\frac{1}{(N+1)}\bigg)^{\frac{1}{N}}
= 1 - \bigg(\frac{2(1-\frac{1}{K})^{(2N)}}{(N+1)(1 + (1-\frac{1}{K})^{N})}\bigg)^{\frac{1}{N}}
\end{align*}

which concludes the proof.
\end{proof}
\begin{proof}[Proof of Proposition \ref{prop:sicgtattack}]
The attack lasts for the entire initialization and one extra round to trigger punishment using an adversarial collision, the length of initialization is $12eK^2\log{T} + K\log{T}$ rounds, hence the cost of attack is $O(K^2\log{T})$. Once a defender receives a collision, all defenders enter punishment phase within $2(N+1)$ rounds, so $T_{\text{punish}} \leq 12eK^2\log{T} + K\log{T} + 2(N+1)$. By Lemma \ref{lem:sicgtpunish}, after round $T_{\text{punish}} + t_p = O(\frac{K^2\log{T}}{(1-\tilde{\alpha})^2\mu_{(K)}})$ with $\tilde{\alpha} = \frac{1 + (1-\frac{1}{K})^{N}}{2}$, each defender $j$ samples actions from the probability distribution ${p}^j$ such that 
\begin{align*}
&    \mathbb{P}\left({\forall \; j : p^j_{(1)} \leq 1 - \bigg(\frac{2(1-\frac{1}{K})^{(2N)}}{(N+1)(1 + (1-\frac{1}{K})^{N})}\bigg)^{\frac{1}{N}}}\right) \geq 1-\frac{3}{T}
\end{align*}

till the end of the time-horizon. Conditioned on this event, during the rounds where defenders sample from constant probability distributions, the probability that the best arm is pulled by some defender in a particular round is at most $\bigg(\frac{2(1-\frac{1}{K})^{(2N)}}{(N+1)(1 + (1-\frac{1}{K})^{N})}\bigg)$.

Therefore, after round  $T_{\text{punish}} + t_p$,  with probability at least $1-\frac{3}{T}$, the expected regret per round is at least $\mu_{(1)}-\mu_{(2)}\bigg(\frac{2(1-\frac{1}{K})^{(2N)}}{(N+1)(1 + (1-\frac{1}{K})^{N})}\bigg)$. The attack succeeds if the initialization phase is successful and the hoeffding bound holds for the defenders which occurs with probability at least $(1 - \frac{N+1}{T})(1 - \frac{3}{T})$.
Hence 
\begin{align*}
    \mathbb{E}R_T \geq (1 - \frac{N+1}{T})(1-\frac{3}{T})(T - T_{\text{punish}} + t_p)(\mu_{(1)}-\mu_{(2)})\bigg(\frac{2(1-\frac{1}{K})^{(2N)}}{(N+1)(1 + (1-\frac{1}{K})^{N})}\bigg) = \Omega(T)
\end{align*}

\end{proof}

\subsubsection{Attack on CNJ, CDJ and CUJ - Algorithms Robust to Jammers}
\label{sec:attackjammers}
CNJ, CDJ and CUJ from \cite{jammers} are algorithms robust to jamming attacks. They assume that the attackers use a fixed jammer algorithm against which they restore logarithmic regret for the defenders. Importantly the feedback model they assume is \textbf{full sensing}, where each defender observes $X_{\pi^{j}(t)}(t)$ and the collision indicators at each round. In our feedback model the true reward from each arm is not revealed.

In the non-distinguishable setting (\textsc{cnj,cuj}), the expected regret can be made $\Omega(T)$ in the extremely simple case where the attacker plays the defender algorithm (\textsc{cnj,cuj}) for $O(\log{T})$ rounds and pulls $\bot$ for the remaining. This is because these algorithms estimate the number of defenders in the first $\tilde{O}(\log{T})$ rounds. Since the attacker plays the same algorithm as defenders, with high probability all defenders estimate the total number of defenders to be $N+1$. During the exploitation phase, the defenders hop over the top $N+1$ arms, during which they a regret of $\sum_{i=1}^{N}{\mu_{(i)}} - \frac{N}{N+1}\sum_{i=1}^{N+1}{\mu_{(i)}} > 0$ per round. Hence $\mathbb{E}{R_T} = \Omega(T)$.  

In the distinguishable collision setting (\textsc{cdj}), the expected regret can be made $\Omega{(T)}$ with $\tilde{O}(\log{T})$ attacks using $K$ attackers. The strategy involves the attackers selecting all the arms orthogonally for $\tilde{O}(\log{T})$ rounds and selecting $\bot$ for the remaining. This algorithm also estimates the number of defenders in the first $\tilde{O}(\log{T})$ rounds, and each defender is forced to evaluate the number of defenders as $round(1 + \frac{\log(0)-\log(0)}{\log(1-\frac{1}{K})})$ (which is not a number) with probability 1. In this case, \textsc{cdj} breaks-down and incurs regret $\Omega{(T)}$.

\newpage
\subsection{Analysis of \textsc{resync}}
\label{sec:analysisRESYNC}
First, we upper bound the number of rounds required for all defenders to have collected sufficient number of observations on each arm to be able to determine the optimal set of arms with high probability. $T_B = T_0 + 2N^2 + N$ denotes the size of each epoch, with $T_0 = 8K\lceil{\log(2K^2T)}/{\Delta^2} \rceil$. A single run of exploration phases and exploitation phases each run for $T_B$ rounds.

\begin{replemma}{lem:explore}
All defenders have collected at least $8\lceil\frac{\log(2K^2T)}{\Delta^2} \rceil$ reliable observations for each arm after at most $T_B + C$ rounds since the start of the game.
\end{replemma}

\begin{proof}
First, all defenders are in exploration phase during the first epoch. If there exists some defender who does not have at least $8\lceil{\log(2K^2T)}/{\Delta^2} \rceil$ observations for each arm during the an epoch, then, all defenders re-enter the exploration phase in the next epoch. This is shown in Lemma \ref{lem:insuff-renter-explr}. Since any defender stores reward observations from previous epochs, the upper bound on the number of rounds required to obtain $\frac{T_0}{K}$ reliable observations on each arm is given by considering the worst case attacker strategy where, the attacker collides on a single arm for a single defender for $C$ rounds. This is the worst case strategy since the defenders collect observations in parallel during sequential hopping. This leads to the upper bound of $T_B + C$ on the number of rounds for all defenders to have sufficient observations on all arms.
\end{proof}
Next we upper bound the number of epochs in which not all defenders are in exploitation phase. For this we state the synchronization lemmas which describe the behaviour of the defenders running \textsc{resync} in presence and absence of adversarial collisions. 

\subsubsection{Synchronization Lemmas}

The first synchronization lemma states that if some defender does not have sufficient observations on all arms by the end of an exploration phase, then, every defender re-enters an exploration phase in the next epoch.

\begin{replemma}{lem:insuff-renter-explr}
Suppose all defenders are in exploration phase in a certain epoch. If there exists some defender who does not have at least $8\lceil{\log(2K^2T)}/{\Delta^2} \rceil$ observations for each arm, then, all defenders re-enter the exploration phase in the next epoch.
\end{replemma} 
\begin{proof}
If there exists some defender $j$ who does not have at least $8\lceil{\log(2K^2T)}/{\Delta^2} \rceil$ observations for each arm, then defender $j$ sets Restart to True before the start of the intra-communication sub phase. If Restart is True for some defender $j$ before communication begins, then during the communication protocol, all other defenders receive a collision, and hence set their variable Restart to True as well. Note crucially that all other defenders receive a collision from defender $j$ during intra-communication phase, since an attacker can induce collisions but cannot remove a collision that occurs between two defenders.
\end{proof}

\begin{replemma}{lem:sync2}
Suppose all defenders are in exploration phase or all defenders are in exploitation phase in the current epoch. If no attacker collides with any defender during this epoch, then all defenders enter exploitation phase in the next epoch.
\end{replemma}
\begin{proof}
In the first case, if no collisions occur from an attacker during an exploration phase, then all defenders have at least $8\lceil{\log(2K^2T)}/{\Delta^2} \rceil$ observations for each arm. Further, due to the design of the sub-phases in exploration, Restart is not set to True by any defender. Therefore Restart remains as False and all defenders progress to exploitation phase in the next epoch.

In the second case, if all defenders are in exploitation phase in a certain epoch and no attacker collides with any defender during the last $N$ rounds of this epoch, then no defender sets Restart to True, and all defenders re-enter exploitation phase in the next epoch.
\end{proof}

The following are the two cases where desynchronization occurs under adversarial collisions. Suppose all defenders are in exploration phase in a certain epoch. If $N_1$ defenders trigger restart and $N_2$ do not, then the $N_1$ defenders re-enter exploration phase while the $N_2$ enter exploitation phase. Similarly if all defenders are in exploitation phase in a certain epoch. If $N_1$ defenders trigger restart and $N_2$ do not, then the $N_1$ defenders enter exploration phase while the $N_2$ re-enter exploitation phase in the next epoch. What is crucial is that all defenders are able to synchronize in these two situations, which the final synchronization lemma proves is true.

\begin{replemma}{lem:sync3}
Suppose $N_1 \geq 1$ defenders are in exploration phase and $N_2 \geq 0$ defenders are in exploitation phase in the current epoch. Then independent of the arms pulled by attackers during the epoch, all defenders enter exploration phase in the next epoch. 
\end{replemma}
\begin{proof}
If $N_1 \geq 1$ defenders are in exploration phase and $N_2$ defenders are in exploitation phase in a certain epoch, then for the defenders in exploration phase, during the sensing phase these defenders experience a collision from the defenders in exploitation phase (since the defenders in exploitation phase must pull Opt$[1]$ every $N$ consecutive rounds), and this collision cannot be reset by the attacker. So the defenders in exploration phase set Restart to True. Further, during the last $N$ rounds of the epoch all defenders in exploration phase pull the Opt$[1]$, and all defenders in exploitation phase experience this collision, and set Restart to True. Since all defenders set Restart to True, all defenders enter the exploration phase in the next epoch.
\end{proof}

Next we upper bound the number of epochs in which not all defenders are in exploitation phase. Note that since we condition on the event that the total number of attacks during the bandit game is at most $C$, at most $C$ epochs can have adversarial collisions. To obtain the bound, we reason over a meta game induced by the bandit game. The synchronization lemmas presented above allow us to reason over the meta game instead of the underlying bandit game. The description of the meta game is presented in the next section.

\subsection{Meta Game and Upper Bound on Number of Bad Phases}
\label{sec:metagame}
Since the defenders use \textsc{resync} which is a deterministic algorithm, we will study the behaviour of the state of the system at each epoch, and how the system reacts to adversarial collisions.
We divide the horizon into $N_b$ epochs of size $T_B$, and reason over the meta game induced by the bandit game, where each $T_B$ sized epoch in the bandit game corresponds to one time step of the meta game. The horizon of the meta game is $N_h = N_b - \lceil1 + \frac{C}{T_B}\rceil$, where $\lceil1 + \frac{C}{T_B}\rceil$, is the number of epochs from the start of the game for all defenders to determined the optimal set of arms with high probability (refer Lemma \ref{lem:explore-dc}).

The description of the meta game is the deterministic Markov decision process $(S,A,R,\delta)$, where,
\begin{align*}
    & A = \{\texttt{N},\texttt{C},\texttt{C'}\} \text{ is the action space },\\
    & S = \{\textsc{explore},\textsc{desync},\textsc{exploit}\} \text{ is the state space }, \\
    & R : S\times A \to \mathbb{R} \text{ is the zero function (We will not consider rewards directly in the MDP)},  \\
    & \delta \text{ is the deterministic transition function given in Table \ref{tab:dynamics}}.
\end{align*}
\begin{table}[H]
\centering
  \label{tab:dynamics}
  \begin{tabular}{rlll}\toprule
    $\delta$ & \texttt{N}    &  \texttt{C}   & \texttt{C'}   \\ \midrule
    \textsc{explore}  & \textsc{exploit} &  \textsc{explore} &  \textsc{desync} \\
    \textsc{desync}  & $-$ &  \textsc{explore}&  $-$  \\
    \textsc{exploit}  & \textsc{exploit}  & \textsc{explore} &  \textsc{desync} \\  \bottomrule
  \end{tabular}
  \caption{Transition table where $\delta$ is the deterministic transition function}
\end{table}

The states $\textsc{explore},\textsc{desync}$ and \textsc{exploit} correspond to the state of the bandit game in a certain epoch where all defenders are in exploration phase, at least one defender is in exploration phase and at least one defender is in exploitation phase, or all defenders are in exploitation phase respectively. Similarly the action space $\{\texttt{N},\texttt{C},\texttt{C'}\}$ corresponds to the type of sequence of actions the attackers can take over epochs. Formally, $\texttt{N}$ corresponds to arm pulls during an epoch that cause no collisions with defenders, $\texttt{C}$ to arm pulls during an epoch that cause all defenders to restart exploration on the next epoch, and $\texttt{C'}$ to arm pulls during an epoch that cause $N_1 : 1<N_1<N$ defenders to restart respectively. That the transition table corresponds exactly to the bandit game when all defenders use \textsc{resync} follows from the synchronization lemmas as show in the following table.

\begin{table*}[h!]
\centering
  
  \label{tab:dynamicsreferences}
  \begin{tabular}{rlll}\toprule
    $\delta$ & \texttt{N}    &  \texttt{C}   & \texttt{C'}      \\ \midrule
    \textsc{explore}  & \textsc{exploit} by Lemma \ref{lem:sync2} &  \textsc{explore} by Definition of \texttt{C} &  \textsc{desync} by Definition of \texttt{C'} \\
    \textsc{desync}  & $-$ &  \textsc{explore} by Lemma \ref{lem:sync3} &  $-$  \\
    \textsc{exploit}  & \textsc{exploit} by Lemma \ref{lem:sync2}  & \textsc{explore} by Definition of \texttt{C} &  \textsc{desync} by Definition of \texttt{C'} \\  \bottomrule
  \end{tabular}
  \caption{References to relevant lemmas proving the transition table corresponds exactly to the bandit game when $N$ defenders use \textsc{resync}.}
\end{table*}

We will not consider any rewards in the meta game. We use the meta game only to determine the number of epochs where all defenders are not in exploration phase, we will not bound the regret here directly.

\addtocounter{figure}{-1}
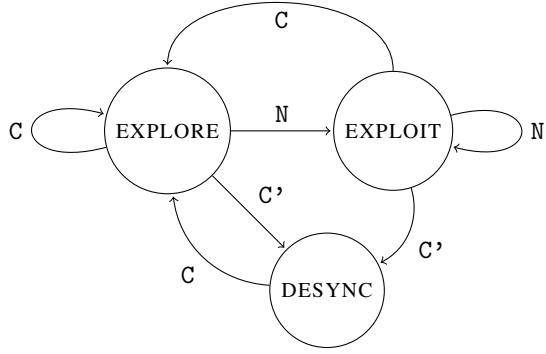
\begin{figure}[h!]
\caption{Transition Dynamics of the Meta game MDP}
\centering
\begin{tikzpicture}[shorten >=1pt,node distance=3cm,on grid,auto] 
   \node[state] (1)   {\textsc{explore}}; 
   \node[state] (2) [below right=of 1] {\textsc{desync}}; 
   \node[state](3) [right=of 1] {\textsc{exploit}};
    \path[->] 
    (1) edge [loop left] node {\texttt{C}}  (1)
    (1) edge  node {\texttt{C'}}  (2)
    (1) edge  node {\texttt{N}}  (3)
    (2) edge  [bend left = 40] node {\texttt{C}}  (1)
    (3) edge [bend right = 90] node {\texttt{C}} (1)
    (3) edge [bend left = 40] node {\texttt{C'}} (2)
    (3) edge [loop right] node {\texttt{N}} (3)
    ;
\end{tikzpicture}
\label{fig:dynamicsApp}
\end{figure}

The following Lemma bounds the number of rounds of the meta game in which the state is not \textsc{exploit} conditional on the the number of times action \texttt{C} or \texttt{C'} being played is at most $C$ (the number of attacks in the bandit game), when the initial state is unknown. In the bandit game, this corresponds to the number of epochs in which not all defenders are in exploitation phase conditional on at most $C$ epochs having adversarial collisions. Since at most $C$ attacks occur during the bandit game, at most $C$ epochs can be under attack, therefore action \texttt{C} or \texttt{C'} are played in at most $C$ time steps of the meta game .

\begin{replemma}{lem:goodstatebound}
Let $N_h$ be the time horizon of the MDP, then \[\sum_{t' = 1}^{N_h}\mathds{1}(A_{t'} = {\normalfont \texttt{C}} \text{ or } A_{t'} = {\normalfont \texttt{C'}}) \leq C \implies \sum_{t' = 1}^{N_h}\mathds{1}(S_{t'} \neq \textsc{exploit}) \leq 1 + 3C\]
\end{replemma}
\begin{proof}
First note that since $S_{t'+1} = \textsc{desync}$ only if $A_{t'} = \texttt{C'}$, the number of times state \textsc{desync} is visited is bounded above, \[\sum_{t' = 1}^{N_h}\mathds{1}(S_{t'} = \textsc{desync}) = 1 + \sum_{t' = 1}^{N_h- 1}\mathds{1}(S_{t' + 1} = \textsc{desync}) \leq \sum_{t' = 1}^{N_h - 1}\mathds{1}(A_{t'} = {\normalfont \texttt{C'}}) \leq C + 1\] Further, since $S_{t'+1} = \textsc{explore}$ only if $S_{t'} = \textsc{desync}$ or $S_{t'} = \textsc{exploit} \text{ and } A_{t'} = \texttt{C}$, the number of visits to state \textsc{explore} can be bounded above as well,
\begin{align*}
 \sum_{t' = 1}^{N_h}\mathds{1}(S_{t'} = \textsc{explore}) & \leq 1 + \sum_{t' = 1}^{N_h-1}\mathds{1}(S_{t'+1} = \textsc{explore})\\
 & \leq 1 + \sum_{t' = 1}^{N_h-1}\mathds{1}(S_{t'} = \textsc{desync} \text{ or } S_{t'} = \textsc{exploit} \text{ and } A_{t'} = \texttt{C})\\
& \leq 1 + \sum_{t' = 1}^{N_h-1}\mathds{1}(S_{t'} = \textsc{desync}) + \sum_{t' = 1}^{N_h-1}\mathds{1}(S_{t'} = \textsc{exploit})\mathds{1}( A_{t'} = \texttt{C})\\
& \leq 1 + C + \sum_{t' = 1}^{N_h-1}\mathds{1}( A_{t'} = \texttt{C}) \leq 1 + C + C - 1 = 2C
\end{align*}

Therefore, \[\sum_{t' = 1}^{N_h}\mathds{1}(S_{t'} \neq \textsc{exploit}) = \sum_{t' = 1}^{N_h}\mathds{1}(S_{t'} = \textsc{explore}) + \sum_{t' = 1}^{N_h}\mathds{1}(S_{t'} = \textsc{desync}) \leq 1 + 3C\]
\end{proof}

\subsection{Proof of Theorem \ref{thm:mainRESYNC}}

For the analysis we assume that the time horizon $T$ is divisible by $T_B$, otherwise the final epoch is of size at most $T_B$ during which the defenders incur regret at most $T_B$. 

The following is Lemma A.2 from \cite{UMLshai}.
\begin{lemma}
\label{lem:shai}
Let $a \geq 1 $ and $b > 0$. Then, $x \geq 4a\log{(2a)} + 2b \implies x \geq a\log{(x)} + b$
\end{lemma}

We will assume the following lower bound on the length of the time-horizon,
\[T \geq \frac{192CK}{\Delta^2}\log{\left(\frac{96CK}{\Delta^2}\right)} + 40CN^2 + 96CK\frac{\log{(2K^2)}}{\Delta^2}\]
So that by Lemma \ref{lem:shai},
\[T \geq \frac{48CK\log{2K^2T}}{\Delta^2} + 20CN^2 \]
which will imply that the final inequality in the Proof of Theorem \ref{thm:mainRESYNC} is non-trivial.

\begin{proof}[Proof of Theorem \ref{thm:mainRESYNC}]
By Lemma \ref{lem:explore}, after at most $T_B + C$ rounds, all defenders have collected at least $8\lceil\frac{\log(2K^2T)}{\Delta^2} \rceil$ observations for each arm. After $T_B + C$ rounds by Lemma \ref{lem:suffobs}, all defenders have a $\frac{\Delta}{2}$-correct ranking with probability at least $\frac{1}{T}$, and hence by Lemma \ref{lem:sortedlist} all defenders have determined the optimal set of arms. Further this set is sorted according to arm index which allows for synchronization during the exploitation phase. Conditioned on the success of the above event, by Lemma \ref{lem:goodstatebound}, all defenders are in exploitation phase in at least $N_b - \lceil1 + \frac{C}{T_B}\rceil - 3C - 1$ epochs of size $T_B$.

Let $\mathcal{E}$ denote the event that all defenders have determined the optimal set of arms. As mentioned above,
\begin{align*}
& \mathbb{P}(\mathcal{E}) \geq 1 - \frac{1}{T} \text{ by Lemma \ref{lem:explore}, Lemma \ref{lem:suffobs} and Lemma \ref{lem:sortedlist} }
\end{align*}
The total regret in the good event $\mathcal{E}$ is 
\[\mathbb{E}(R_T|\mathcal{E}) \leq T_B\left\lceil1 + \frac{C}{T_B}\right\rceil + T_B(3C + 1) + C\]
where $T_B\lceil1 + \frac{C}{T_B}\rceil$ is the maximum number of rounds required for sufficient observations on each arm, $T_B(3C + 1)$ is the maximum regret accumulated when at least one defender is not in exploitation phase and $C$ is the maximum regret accumulated when all defenders are in exploitation phase and face adversarial collisions from attackers.

In summary, the total expected regret conditioned on the number attacks being $C$ is,
\begin{align*}
\mathbb{E}R_T & = \mathbb{E}\left(R_T|\mathcal{E}\right)\mathbb{P}\left(\mathcal{E}\right) + \mathbb{E}\left(R_T|\mathcal{E}^{\mathsf{c}}\right)\mathbb{P}\left(\mathcal{E}^{\mathsf{c}}\right) \\
              & \leq \mathbb{E}\left(R_T|\mathcal{E}\right)\cdot 1 + T \cdot \frac{1}{T} \\
              & \leq \left(T_B\left\lceil1 + \frac{C}{T_B}\right\rceil + T_B(3C + 1) + C\right) + 1 \leq 3T_B + 2C + 3CT_B + 1\\
              & \leq 24CK\frac{\log{(2K^2T)}}{\Delta^2} + 9CN^2+ 2C + 24K\frac{\log{(2K^2T)}}{\Delta^2} + 9N^2 + 1 \\
              & \leq O\left( N^2C + KC\frac{\log(K^2T)}{\Delta^2} \right)
\end{align*}

\end{proof}
\newpage
\subsection{Complete Description of \textsc{resync2}}
\label{app:RESYNC2-desc}
We here describe in detail the pseudocode for the \textsc{resync2} algorithm. The algorithm \textsc{resync2}, consists of three phases that run sequentially, namely, initialization, exploration and exploitation. 

\begin{wrapfigure}{L}{0.46\textwidth}
\begin{minipage}[t]{0.46\textwidth}
  \vspace{0pt}
\begin{algorithm}[H]
\caption{\textsc{resync2}}
\SetKwInOut{Input}{Input}
\KwIn{$T$ (horizon)}
Restart $\gets $ True\\
$Opt \gets \emptyset, \tilde{\mu}_i \gets 0, o_i \gets 0, s_i \gets 0 $\\
k $\gets$ Orthogonalization($[K]$)\\
$(n,j) \gets$ Estimate-Defenders($k$)\\

Opt $\gets$ ExplorationDC($N, \forall i \; o_i,s_i$)\\

\While{$t \leq T$}{
Set $k = t+j$ (mod $N$) and Pull Opt$[k]$
}
\end{algorithm}
\end{minipage}
\end{wrapfigure}

The purpose of the initialization phase is to break symmetry between the players, allow them to estimate the total number of defenders, and assign distinct internal ranks between the defenders. The exploration phase progresses in epochs of size $2K$. Each epoch of exploration, has two phases of size $K$. In the first phase, i.e. the first $K$ rounds of the epoch, the defenders sequentially hop arms to obtain one observation on each arm.  During the next phase, the defenders communicate (via arm collisions) whether or not to restart the current epoch based on whether adversarial collisions occurred during the first phase. If some defender witnessed an adversarial collision during the first phase of the epoch then all defenders disregard the epoch, and enter a new one. The defenders explore until there are sufficient number of uncorrupted epochs, after which they are able to determine the optimal set of arms with high probability. Thereafter, all defenders enter the exploitation phase and sequentially hop over the optimal set of arms.\\

The first part of initialization consists of orthogonalization, first introduced as the Musical-Chairs subroutine in \cite{MC}. The proof that orthogonalization is successful under any adversarial collisions is presented in Lemma \ref{lem:orth}. The initialization phase then consists of a second procedure. Its purpose is to estimate the number of defenders and to assign different ranks in $[N]$ to all defenders. The proof that the subroutine works conditional on the success of orthogonalization and under any adversarial collisions in presented in Lemma \ref{lem:estimate-defenders}.

\begin{minipage}[t]{0.46\textwidth}
  \vspace{0pt}
\begin{subroutine}[H]
\caption{Estimate-Defenders}
\SetKwInOut{Input}{Input}
\KwIn{$k \in [K]$ (external rank)}
\KwOut{$N$ (Estimated number of defenders), $j$ (internal rank)}
$N \gets 1, j \gets 1, \pi \gets k$\\
\For{$2k$ rounds}{
Pull $\pi$\\
\If{$\eta^{D}_{\pi} = 1$}{$N \gets N + 1$, $j \gets j + 1$}
}
\For{$2(K-k)$ rounds}{
$\pi \gets \pi + 1$ (mod $K$) and Pull $\pi$ \\
\If{$\eta^{D}_{\pi} = 1$}{$N \gets N + 1$}
}

\Return{$(n,j)$}
\end{subroutine}
\end{minipage}
\hfill
\begin{minipage}[t]{0.46\textwidth}
  \vspace{0pt}
  \begin{subroutine}[H]
\caption{Orthogonalization}
\SetKwInOut{Input}{Input}
\KwIn{$A$ (set of arms)}
\KwOut{\normalfont \texttt{Rank}}
\normalfont \texttt{Rank} $\gets -1$\\
\For{$\lceil K \log{T} \rceil$ times steps}{
\eIf{\normalfont \texttt{Rank} $= -1$}{
Pull $k \sim \mathcal{U}(K)$\\
\If{$\eta_k = 0$ }{
\texttt{Rank} $\gets k$
}
}
{Pull \texttt{Rank}}
}
\Return{\normalfont {\texttt{Rank}}}
\end{subroutine}
\end{minipage}

Next we consider the exploration phase in which each defender stays until all defenders have sufficient number of observations to determine the optimal ranking of arms with high probability. That this property holds is shown in Lemma \ref{lem:explore-dc}.

\newpage
\begin{subroutine}[H]
\caption{ExplorationDC}
\SetKwInOut{Input}{Input}
\KwIn{$N, \forall i \; o_i,s_i$}
\KwOut{Opt}
Restart $\gets$ False, epochNumber $\gets 1$\\
$T_0 \gets 16K\lceil{\log(2K^2T)}/{\Delta^2} \rceil$\\
\While{epochNumber $\leq \frac{T_0}{2K}$ }{
\For(\tcp*[f]{Sequential hopping phase}){$K$ rounds}{
Restart $\gets$ False\\
Pull $k = t+j$ (mod $K$) and receive $\eta^{D}_k , \eta^{A}_k$ and $r_k(t)$ \\
$\eta_k \gets \eta^{D}_k \lor \eta^{A}_k$\\
\eIf{$\eta_k = 0$ }{
$o_k \gets o_k + 1$\\
$s_k \gets s_k + r_k(t)$
}{
Restart $\gets$ True
}
}
\For(\tcp*[f]{Communication phase}){$K$ rounds}{
\eIf{{\normalfont Restart}}{Pull $1$}{
Pull $k = t+j$ (mod $K$)\\
\If{$\eta^{D}_k = 1$}{Restart $\gets$ True}
}
}
\If{{\normalfont not Restart}}{epochNumber $\gets$ epochNumber + 1}
}

$\forall i : \tilde{\mu}_i = s_i/o_i$\\
Opt $\gets$ List of $N$ best empirically performing arms sorted according to arm index.\\
\Return{{\normalfont Opt}}
\end{subroutine}

\subsection{Proof of Theorem \ref{thm:RESYNC2} - Analysis of \textsc{resync2}}
\label{sec:analysisRESYNC2}

\begin{lemma}
\label{lem:orth}
At time $T_{orth} = \lceil K\log{T} \rceil$ all defenders pull different arms in $[K]$ with probability at least $1 - \frac{N}{T}$
\end{lemma}
\begin{proof}
As there is at least one arm that is not played by all the other players at each time step, the
probability of having no collision at time $t$ for a single defender $j$ is lower bounded by $\frac{1}{K}$. It thus holds, that for a single player, her probability to encounter only collisions until time $T_{orth}$:
\begin{align*}
    \mathbb{P}\left(\forall t \leq T_{orth} : \eta^{j}(t) = 1 \right) \leq \left(1- \frac{1}{K}\right)^{T_{orth}} \leq \frac{1}{T}
\end{align*}

A union bound over the $N$ players yields the result.
\end{proof}

\begin{lemma}
\label{lem:estimate-defenders}
If all defenders are in an orthogonal position at time $K\log{T}$ then at time $K\log{T} + 2K$ all defenders have estimated the total number of defenders to be $N$, and have distinct internal ranks in $[N]$.
\end{lemma}
\begin{proof}
Indexing the $2K$ rounds of the Estimate-Defenders protocol using $t'$, the arm pulled by the defender with external rank $k$ at round $t'$ is given by
\[
\pi_{k}(t')=
\begin{cases}
k \; \text{ if } \; 1 \leq t' \leq 2k\\
t' - k \; \text{ if }\;  2k \leq t'\\
\end{cases}
\]

Consider any defender with external rank $k$ with estimated values for $N$ and $j$ initially equal to $1$ (call them $\hat{N}$ and $\hat{j}$). Suppose there are $x$ defenders with external rank less that $k$. Any defender with external rank $i : i<k$, pulls arm $k$ at round $k+i$. Hence by round $2k$, $x$ defenders have pulled arm $k$, at which point $\hat{N} = x + 1$ and $\hat{j} = x+1$, thereafter $\hat{j}$ is not updated. Suppose there are $y$ defenders with higher external rank than $k$. For any defender with external rank $i : i >k$, the defender with external rank $k$ pulls their external rank $i$ at round $k+i < 2i$. Hence by round $2K$,  $\hat{N} = y + x + 1$ and $\hat{j} = x+1$. Since for any defender $x+y = n-1$, each defender estimates $\hat{N}=N$ and since for no two defenders the value of $x$ can coincide, the defenders have distinct internal ranks at round $2K$.

\end{proof}

For the following, we condition on the success of the orthogonalization protocol. We let $T_0 = 16K\frac{\log(2K^2T)}{\Delta^2}$, and condition on the number of collisions $C$.

\begin{replemma}{lem:explore-dc}
After at most $2KC + T_0$ rounds of the ExplorationDC protocol, all defenders have collected at least $8\lceil\frac{\log(2K^2T)}{\Delta^2} \rceil$ observations for each arm. 
\end{replemma}
\begin{proof}
The exploration protocol progresses in epochs of size $2K$. At least $\frac{T_0}{2K}$ epochs must be successful for the exploration protocol to terminate. An epoch is successful if no defender receives a collision from an attacker during the first $K$ rounds of the epoch, at which point all defenders move to the next epoch. If some defender received a collision from an attacker during the first $K$ rounds then, during the next $K$ rounds of the epoch all defenders are notified by that defender through deliberate collisions and collision feedback, at which point all defenders repeat the epoch.

Since collisions occur from all attackers at most $C$ times in total, they can occur in at most $C$ epochs. Since an epoch is reset if and only if collision occurred from an attacker at the first $K$ rounds of that epoch, and the epoch length is $2K$, the total duration of the exploration phase is at most $2KC + T_0$.
\end{proof}

\begin{proof}[Proof of Theorem \ref{thm:RESYNC2}]
First due to the success of the initialization protocol, by Lemma \ref{lem:orth} at time $\lceil K\log{T}\rceil$ all defenders pull different arms with probability $1 - \frac{N}{T}$. Conditioned on this event the Estimate-Defenders protocol runs for $2K$ rounds and is successful by Lemma \ref{lem:estimate-defenders}. Further,  by Lemma \ref{lem:explore-dc}, after at most $2KC + T_0$ rounds of the ExplorationDC protocol, all defenders have collected at least $8\lceil\frac{\log(2K^2T)}{\Delta^2} \rceil$ observations for each arm. By Lemma \ref{lem:suffobs} all defenders have a $\frac{\Delta}{2}$-correct ranking with probability at least $\frac{1}{T}$, and hence by Lemma \ref{lem:sortedlist} all defenders have determined the optimal set of arms. As in \textsc{resync}, this optimal set is sorted according to arm index which allows for synchronization during the exploitation phase. At any round of the exploitation phase, conditioned on previous events, all defenders pull the top $N$ arms in an orthogonal fashion. The regret due to collisions in this phase is at most $C$.

Let $\mathcal{E}_1$ denote the event that initialization is successful. Let $\mathcal{E}_2$ denote the event that after $2KC + T_0$ rounds since initialization completed, all defenders have collected at least $8\lceil\frac{\log(2K^2T)}{\Delta^2} \rceil$ uncontaminated observations for each arm.

\begin{align*}
& \mathbb{P}(\mathcal{E}_1) \geq 1 - \frac{N}{T} \text{ by Lemma \ref{lem:orth} and Lemma \ref{lem:estimate-defenders}  and } \mathbb{P}(\mathcal{E}_2 | \mathcal{E}_1) \geq 1 - \frac{1}{T} \text{ by Lemma \ref{lem:explore-dc}}\\
& \implies \mathbb{P}(\mathcal{E}_1 \cap \mathcal{E}_2) = \mathbb{P}(\mathcal{E}_2 | \mathcal{E}_1)\cdot \mathbb{P}(\mathcal{E}_1) \geq 1 - \frac{N+1}{T}
\end{align*}

The total regret in the good event $\mathcal{E}_1 \cap \mathcal{E}_2$ is 
\[\mathbb{E}(R_T|\mathcal{E}_1 \cap \mathcal{E}_2) \leq K\log{T} + 2K + 2KC + 16K\frac{\log(K^2T)}{\Delta^2} + C\]

where $K\log{T} + 2K$ is the length of initialization, $2KC + 16K\frac{\log(K^2T)}{\Delta^2}$ is the maximum number of rounds required for sufficient observations on each arm, and $C$ is the maximum regret accumulated in an exploitation phase where defenders only pull the top $N$ arms.

In summary, the total expected regret conditioned on the number attacks being $C$ is,
\begin{align*}
\mathbb{E}R_T & = \mathbb{E}\left(R_T|\mathcal{E}_1 \cap \mathcal{E}_2\right)\mathbb{P}\left(\mathcal{E}_1 \cap \mathcal{E}_2\right) + \mathbb{E}\left(R_T|(\mathcal{E}_1 \cap \mathcal{E}_2)^{\mathsf{c}}\right)\mathbb{P}\left((\mathcal{E}_1 \cap \mathcal{E}_2)^{\mathsf{c}}\right) \\
              & \leq \mathbb{E}\left(R_T|\mathcal{E}_1 \cap \mathcal{E}_2\right)\cdot 1 + T \cdot \frac{N+1}{T} \\
              & \leq K\log{T} + 2K + 2KC + 16K\frac{\log(K^2T)}{\Delta^2} + C + N + 1 \leq O\left( KC + K\frac{\log(K^2T)}{\Delta^2} \right)
\end{align*}

\end{proof}

\subsection{Proofs of Lower Bounds}

In this Section we consider lower bounds on expected regret of the defenders, when the number of attacks is $C$, and will show that the expected regret of the defenders has a linear dependence on number of attacks. The following proposition states that conditional on the number of attacks being $C$, if the defenders use any MP-MAB algorithm then they must incur regret $\Omega({C})$. 
\begin{proposition}
Conditioned on the number of attacks being $C$, the expected regret of the defenders using any MP-MAB algorithm is at least $\Omega({C})$.
\end{proposition}
\begin{proof}
If there was a collision from some attacker at some round $t$, then the expected sum of rewards of the defenders that round is at most $\sum_{i=1}^{N-1}{\mu_{(i)}}$. Since there are exactly $C$ attacks, the regret accumulated during the rounds where attack occurred is at least \[{C}\cdot (\sum_{i=1}^{N}{\mu_{(i)}} - \sum_{i=1}^{N-1}{\mu_{(i)}}) = C\cdot{\mu_{(N)}} = \Omega\left({C}\right)\]

and so the total expected  regret is $\Omega({C})$.
\end{proof}

The next theorem provides a simple attack, against which the defenders incur expected regret $\Omega(\frac{NC}{K})$ by using any MP-MAB algorithm .

\begin{reptheorem}{thm:lowerbound1}
There exists an attacker with expected number of attacks at most $C$, for which any MP-MAB algorithm has expected regret $\Omega(\frac{NC}{K})$. 
\end{reptheorem}
\begin{proof}
The attacker samples an arm $k \sim \mathcal{U}(K)$, pulls $k$ for $C$ rounds, and pulls $\bot$ for the remaining rounds. Clearly the number of collisions any defender will face from the attacker is at most $C$. The sampled arm is one of the $N$ best arms with probability $\frac{N}{K}$. Under this event, during the first $C$ rounds any algorithm has per-round regret at least $\mu_{(N)}-\mu_{(N+1)}$. So, the expected regret over $T$ rounds under this attack is at least $C\cdot(\mu_{(N)}-\mu_{(N+1)})\cdot\frac{N}{K} = \Omega\left(\frac{NC}{K}\right)$.
\end{proof}

Clearly, if an attacker knew the optimal set of arms, then by pulling an optimal arm for $C$ rounds, the regret of defenders is forced to be $\Omega(C)$. 

\newpage
\section{Experiments}
\label{app:experiments}
For our experiments we implement the \textsc{resync,mc} and \textsc{sic-mmab} algorithm in the non-distinguishable collision sensing case and \textsc{resync2} and \textsc{cdj} in the distinguishable case. For each experimental setup and algorithm, with the considered regret values averaged over $20$ runs, and plotted the average and standard deviation of the resulting regret (the standard deviation is shown as error bands encompassing the average regret). Experiments for \textsc{resync2} and \textsc{cdj} appear in Appendix \ref{app:experiments}.

\begin{figure}[H]
\centering
    \includegraphics[width=0.4\textwidth]{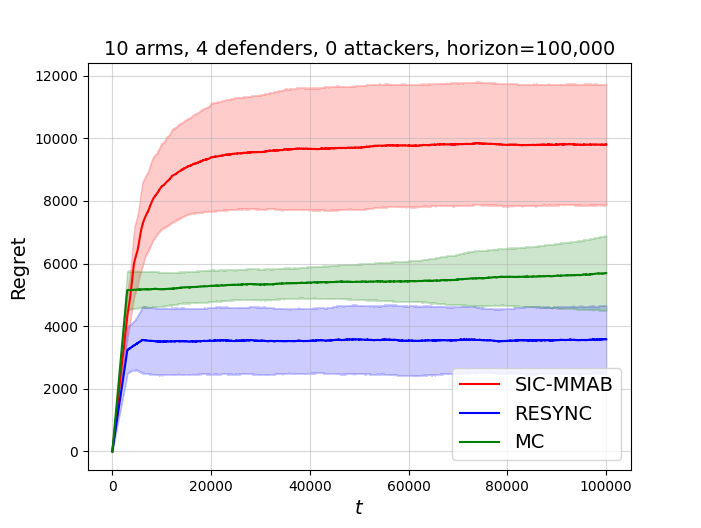} 
    \caption{Performance comparison of \textsc{resync}, \textsc{sic-mmab} and \textsc{mc} with no attackers. The figure shows the evolution of cumulative regret over time.}\label{fig:noattack}

\end{figure}
In all experiments the gap $\Delta = \mu_{(N)} - \mu_{(N+1)}$ is at least $0.05$. The experiments are run with Bernoulli distributions with means chosen uniformly at random from $[0,1]$ (with $\Delta \geq 0.05$). Further Assumption \ref{assm:cmn-knwlg} is applied to all algorithms, removing the necessity to estimate the number of defenders and their ranks online. For the \textsc{mc} and \textsc{resync} algorithms we set $T_0$ to be $3000$ in all experiments as suggested by \cite{MC}.  

We begin with a simple scenario comparing the performances of the algorithms under no attack. Figure \ref{fig:noattack} compares the evolution of cumulative regret, with the following problem parameters: $K = 10, N = 5, m = 0, T = 10^5$.
\begin{figure}[H]
    \centering
    \includegraphics[width=0.4\textwidth]{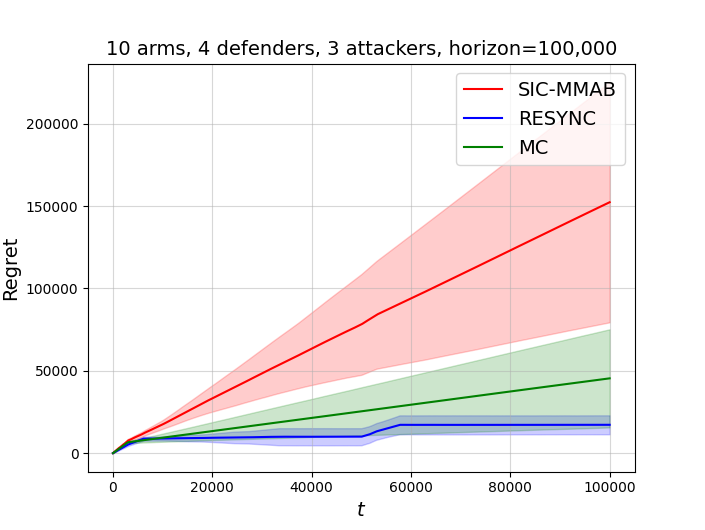} 
    \caption{Performance comparison of \textsc{resync}, \textsc{sic-mmab} and \textsc{mc} with 2 attackers. The figure shows the evolution of cumulative regret over time.}\label{fig:resyncattack}
\end{figure}
Second we compare the performances of the algorithms under attack. Figure \ref{fig:resyncattack} compares the evolution of cumulative regret, with the following problem parameters: $K = 10, N = 5, m = 2, T = 10^5$. The attackers are centralized and do the following, first they select $M$ distinct arms, then pull these arms for $T_0$ rounds. After $T_0$ rounds they stay silent and pull $\bot$. After round $5\times10^4$ they again, pull these arms for $T_0$ rounds, and stay silent afterwards.

As can be noted from Figure \ref{fig:resyncattack}, under this attack the expected regret scales linearly in $t$ for both \textsc{sic-mmab} and \textsc{mc} due to incorrect estimation of the optimal set of arms in \textsc{mc} and \textsc{sic-mmab}, and \textsc{sic-mmab} in particular suffers from desynchronization due to attacks occurring during communication rounds. After the first $6\times10^3$ rounds, all defenders running \textsc{resync} obtain sufficient observations for each arm and begin exploiting which lasts up-to round $5\times10^4$, causing no added regret during rounds $6\times10^3$ to $5\times10^4$. After this due to attack, the defenders restart and enter exploration which causes the linear increase in regret during the rounds $5\times10^4$ to $5.6\times10^4$ after which the defenders are able to synchronize and enter exploitation again until the end of the time horizon causing no added regret after round $5.6\times10^4$.\\

For the second set of experiments in the distinguishable collision sensing setting, we implement the \textsc{resync2} algorithm \textsc{cdj} algorithm for comparison. For each experimental setup and algorithm, with the considered regret values averaged over $20$ runs, and plotted the average and standard deviation of the resulting regret (the standard deviation is shown as error bands encompassing the average regret). Experiments for the distinguishable collision sensing setting comparing  appear in \ref{app:experiments}.
In all experiments the gap $\Delta = \mu_{(N)} - \mu_{(N+1)}$ is at least $0.05$. The experiments are run with Bernoulli distributions with means chosen uniformly at random from $[0,1]$ (with $\Delta \geq 0.05$). For the \textsc{cdj} algorithm we set $T_C = 5000$ and for \textsc{resync2} we set $T_0 = 5000$.  

We begin with a simple scenario comparing the performances of the algorithms under no attack. Figure \ref{fig:noattackcnjresync} compares the evolution of cumulative regret, with the following problem parameters: $K = 10, N = 5, m = 0, T = 10^5$.


\begin{figure}[H]

\centering
    \includegraphics[width=0.4\textwidth]{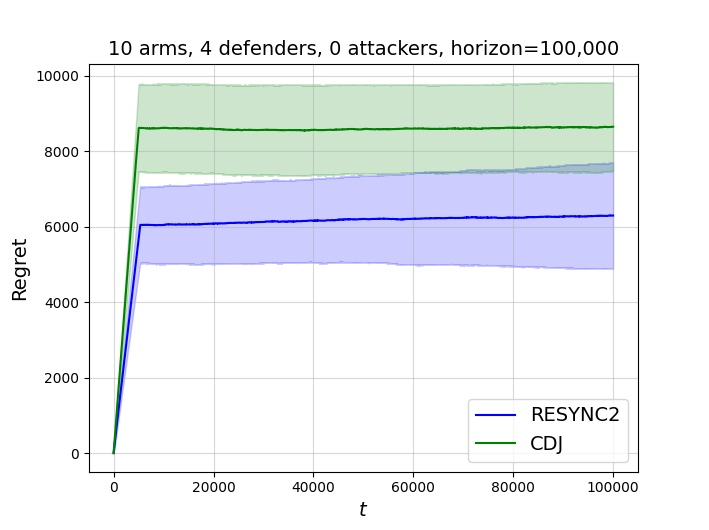} 
    \caption{Performance comparison of \textsc{resync2} and \textsc{cdj} with no attackers. The figure shows the evolution of cumulative regret over time.}\label{fig:noattackcnjresync}

\end{figure}
Second we compare the performances of the algorithms under attack. Figure \ref{fig:cnjresyncattack} compares the evolution of cumulative regret, with the following problem parameters: $K = 10, N = 5, m = 4, T = 10^5$. The attackers are decentralized and play arms uniformly at random for the first $5000$ rounds, and thereafter pull $\bot$ and remain silent.
\begin{figure}[H]
    \centering
    \includegraphics[width=0.4\textwidth]{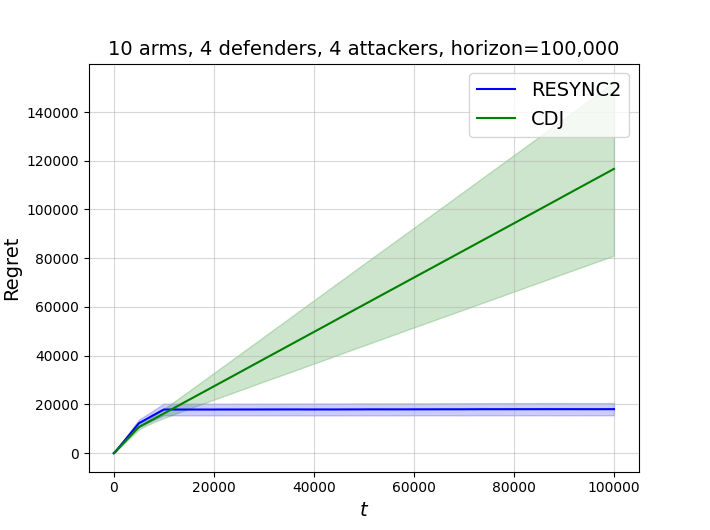} 
    \caption{Performance comparison of \textsc{resync2}, and \textsc{cdj} with 4 attackers. The figure shows the evolution of cumulative regret over time.}\label{fig:cnjresyncattack}
\end{figure}

As can be noted from Figure \ref{fig:cnjresyncattack}, under this attack the expected regret scales linearly in $t$ for \textsc{cnj} due to incorrect estimation of the number of defenders. In order for \textsc{cdj} to estimate the number of defenders correctly, the attackers must play a fixed jammer algorithm as specified in \cite{jammers}, where the attackers must pull $M$ orthogonally and uniformly at random. By mildly deviating from this strategy, \textsc{cdj} can no longer estimate the total number of defenders robustly. On the other hand, \textsc{resync2} incurs constant regret after round $10000$.

\section{Removing the Assumption that $\Delta$ is known}
\label{app:removedelta}
In this section we highlight how the techniques from \textsc{resync2} can be combined with \textsc{sic-mmab} in order to remove all prior knowledge except $T$ in the distinguishable collision sensing setting. The resulting algorithm will have the structure of \textsc{sic-mmab} along with the robust phases presented in \textsc{resync2}. The resulting algorithm will have an initialization phase along with exploration-communication phases. Each exploration phase will contain a sequential hopping and intra-communication phase. There will also be a separate communication phase outside the exploration phase used to communicate arm-statistics with other defenders. For the remainder of this section, we use $N_p, K_p$ to denote the active number of defenders in phase $p$ and active number of arms in phase $p$ respectively.\\
\textbf{Initialization.} The initialization phase will require the following substitutions of subroutines in \textsc{sic-mmab}. The Musical Chairs subroutine from \textsc{sic-mmab} must be substituted with the Orthogonalization subroutine from \textsc{resync2} and the Estimate\_M subroutine from \textsc{sic-mmab} must be substituted with the Estimate-Defenders subroutine from \textsc{resync2}. This must be done to ensure the initialization phase is robust to adversarial collisions.\\
\textbf{Communication.} Next as pointed out in Section \ref{sec:RESYNC2}, communication can be made robust (i.e. defenders can exchange bits in the presence of attackers) by relying on the collision indicator $\eta^{D}_k$ instead of $\eta_k$. Therefore the communication phase from \textsc{sic-mmab} can be preserved by substituting each occurrence of $\eta_k$ in the algorithm with $\eta^{D}_k$.\\
\textbf{Exploration.} Finally the exploration phase will need to inherit the intra-communication sub phase from \textsc{resync2} in order to ensure robust exploration. That is, the outcome of sequential hopping (whether one reliable observation was received for each arm during sequential hopping) must be communicated through forced collisions in the intra-communication phase. In the $p$-th exploration-communication phase, the exploration phase will progress in epochs of size $2K_p$, and until there are $2^p$ successful epochs. By successful epoch we mean that each defender recieves one reliable observation for each active arm in that epoch.  During the first $K_p$ rounds of an epoch the defenders sequentially hop the active set of arms and the next $K_p$ rounds of exploration is an intra-communication phase as in \textsc{resync2}. The following exploration phase must be substituted in \textsc{sic-mmab}. 
\begin{subroutine}[H]
\caption{Exploration Phase}
$\pi \gets j$-th active arm, epochNumber $\gets$ 0\\ 
\While{epochNumber $\leq 2^p$}{
\For{$K_p$ rounds}{
Restart $\gets$ False\\
$\pi \gets \pi + 1$ (mod $[K_p]$)\\
Pull $\pi$ and receive $\eta^{D}_\pi , \eta^{A}_\pi, r_\pi(t)$\\
$\eta_\pi \gets \eta^{D}_\pi \lor \eta^{A}_\pi$\\
\eIf{$\eta_\pi = 0$}{
$s[\pi] \gets s[\pi] + r_\pi(t)$
}{
Restart $\gets$ True}
}

\For{$K_p$ rounds}{
\eIf{Restart}{
Pull $K_p[1]$
}{
$\pi \gets \pi + 1$ (mod $[K_p]$)\\
Pull $\pi$ and receive $\eta^{D}_\pi$\\
\If{$\eta^{D}_\pi = 1$}{
Restart $\gets$ True
}
}
}

\If{not Restart}{
epochNumber $\gets$ epochNumber $ + 1$
}
}
\end{subroutine}

Finally the statistics are updated according to the description in \textsc{sic-mmab}. The technical innovations required to save \textsc{sic-mmab} from attackers are to incorporate our robust initialization, exploration and communication phases that we have described here which were inherited from our algorithm \textsc{resync2}. For the expected regret an additional factor of $O(CK)$ is incurred to the regret bound in Theorem 1 from \cite{sicmmab} due to the robust exploration phase detailed above (for details of the bound $O(CK)$ refer Proof of Theorem \ref{thm:RESYNC2}). The final bound on expected regret of the resulting algorithm is 
$O\left(CK + \sum_{k > N}{\frac{\log{T}}{\mu_{(N)} - \mu_{(k)}}} + MK\log{T}  \right)$ where $C$ is the number of collisions from attackers. For contrast, the expected regret of \textsc{sic-mmab} in the presence of no attackers is 
$O\left(\sum_{k > N}{\frac{\log{T}}{\mu_{(N)} - \mu_{(k)}}} + MK\log{T}  \right)$.

\end{document}